%% file: 0Stream3Dv2.tex
\documentclass[lettersize,journal]{IEEEtran}
\usepackage{amsmath}
\usepackage{amsthm}
\usepackage{amsfonts}
\usepackage{amssymb}
\usepackage{algorithmic}
\usepackage[hyphens]{url}
\usepackage{array}
\usepackage[format=plain,labelformat=simple,labelsep=period,font=small,compatibility=false]{caption}
\usepackage[font=footnotesize,skip=3pt,subrefformat=parens]{subcaption}
\usepackage[numbers,sort&compress]{natbib}
\usepackage{textcomp}
\usepackage{stfloats}
\usepackage{graphicx}
\usepackage{soul}
\usepackage{booktabs}
\usepackage{overpic}
\usepackage{bbm}
\usepackage[accsupp]{axessibility}
\usepackage{multirow}
\usepackage{pifont}
\usepackage{threeparttable}
\usepackage{colortbl}
\usepackage[ruled,vlined]{algorithm2e}
\usepackage{algorithmic}
\usepackage{hyperref}

\newtheorem{theorem}{Theorem}
\newtheorem{lemma}{Lemma}

\usepackage{caption}
\newcommand{\eg}{{\em e.g.,~}}           % e.g.
\newcommand{\ie}{{\em i.e.,~}}           % i.e.
\newcommand{\etc}{{\em etc.~}}           % etc.  

\begin{document}

% \title{Streaming Zero-Shot 3D Scene Understanding with Local Multi-View Geometric-Semantic Fusion and Local-Historical Point Cloud Manifold Propagation}
% \title{Stream3Dv2: A Comprehensive Geometric-Semantic Fusion Framework for Open-Vocabulary Streaming Zero-Shot 3D Scene Understanding}
\title{Stream3Dv2: Geometric-Semantic Fusion Enhanced Streaming Zero-Shot 3D Scene Understanding}

\author{
Jie Xu,
Na Zhao
\IEEEcompsocitemizethanks{
\IEEEcompsocthanksitem Jie Xu and Na Zhao (\{jie\_xu2,~na\_zhao\}@sutd.edu.sg) are with the ISTD Pillar, Singapore University of Technology and Design, Singapore 487372.
\IEEEcompsocthanksitem This research was supported by the Ministry of Education, Singapore, under its MOE Academic Research Fund Tier 2 (MOE-T2EP20124-0013).
}
}



\maketitle

\begin{abstract}
Recently, open-vocabulary zero-shot 3D scene understanding using vision foundation models has emerged as a promising alternative to data-intensive supervised methods. However, deploying these models in real-world scenarios is severely hindered by their inability to efficiently handle streaming RGB-D inputs and their inherent vulnerability to noise 2D segmentation masks. To address these critical limitations, we propose Stream3Dv2, a novel training-free framework designed for robust streaming 3D perception. Stream3Dv2 processes sequential data through an original nested local-to-historical architecture, capturing multi-view consistency while circumventing the high computational overhead so as to support timely responses. At its core, we introduce a comprehensive geometric-semantic fusion mechanism that resolves geometric noise and semantic ambiguity by explicitly utilizing semantic guidance and formulating 3D segmentation as solving point-and-set merging and partitioning problems. Furthermore, we present an innovative manifold-distance-based point cloud refinement strategy. This approach leverages local manifold graphs for point-to-manifold optimization that mitigates the boundary delineation failures caused by Euclidean-distance metrics, and employs geometric bounding boxes to dynamically activate and update historical instances for achieving rapid manifold-to-manifold refinement. Extensive experiments on public datasets demonstrate that Stream3Dv2 consistently outperforms existing baselines in foundational open-vocabulary streaming 3D segmentation and detection. Finally, we show that integrating our framework with an LLM-based agent enables advanced language-driven 3D scene understanding, underscoring its potential for open-world embodied intelligence. Code will be updated at \url{https://github.com/SubmissionsIn/Stream3D}.
\end{abstract}


\section{Introduction}\label{sec:intro}
\begin{figure*}[t]
\centering
\includegraphics[width=1\linewidth]{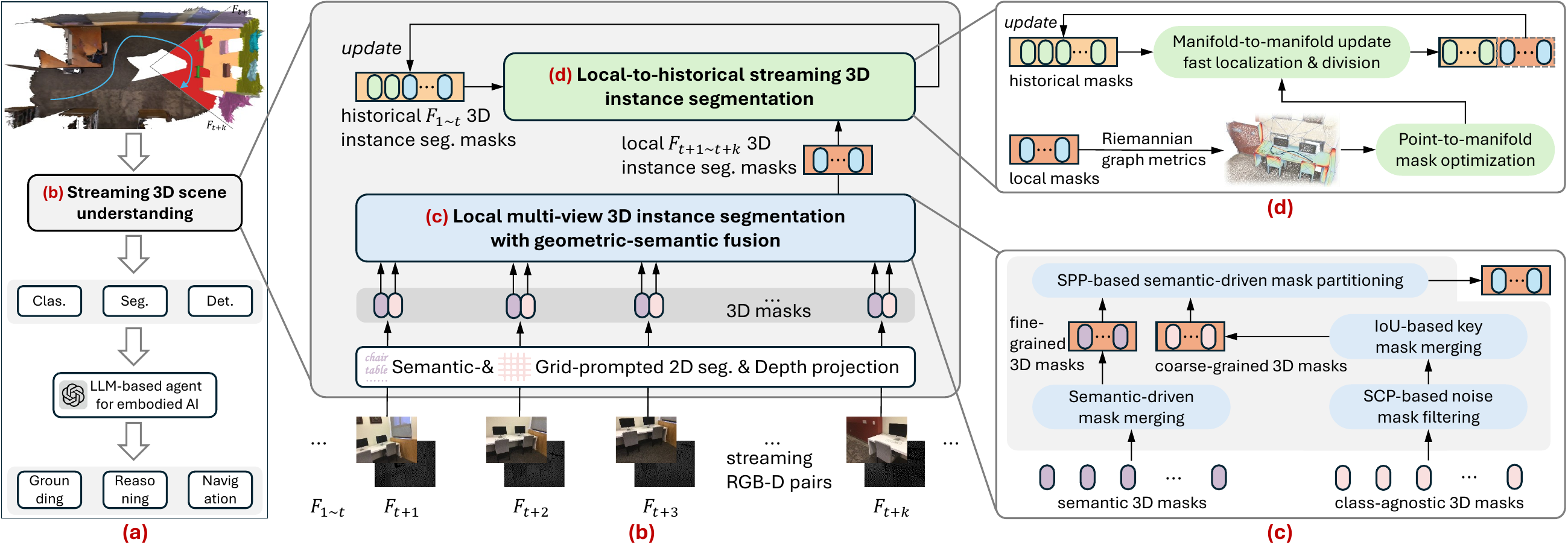}
\caption{\textbf{Overview of Stream3Dv2}: (a) Streaming 3D scene understanding for embodied AI applications. (b) We leverage semantic\&grid-prompted 2D segmentation and integrate \textit{local multi-view} and \textit{local-to-historical} modules to enable streaming zero-shot 3D instance segmentation. Specifically, (c) \emph{local multi-view} module conducts class-agnostic noise mask filtering and semantic-driven mask merging\&partitioning with set covering problem (SCP) and set partitioning problem (SPP) based strategies. Local geometric-semantic fusion produces robust 3D semantic segmentation. Then, (d) \emph{local-to-historical} module refines 3D masks' manifolds by assigning point-to-manifold on Riemannian graph metrics, followed by the detection-based efficient 3D instance localization\&division to achieve manifold-to-manifold 3D mask updates.
}\label{fig:pipeline}
\end{figure*}

\IEEEPARstart{S}{treaming} 3D scene understanding is a fundamental step for 3D computer vision~\cite{geiger20133d,schult2023mask3d,lei2023recent} and embodied intelligence~\cite{seita2023toolflownet,kong2025multi,wang2026affordbot}, aiming to sequentially identify discriminative instances by semantic information across a specific 3D scene, as displayed in Fig.~\ref{fig:pipeline}(a). To produce this, existing approaches predominantly rely on employing training-based deep models~\cite{qi2017pointnet,zhao2021point} which take point clouds as input and are optimized on carefully curated datasets, \eg ScanNet~\cite{dai2017scannet}, MatterPort3D~\cite{Matterport3D}. Despite progress, these methods depend heavily on supervised learning using paired point clouds and human-annotated labels. Consequently, their generalization ability is constrained by costly and limited real-world data annotations, often leading to degraded performance when encountering unseen instances or scenes with domain gaps~\cite{Huynh2021OpenVocabularyIS,Vibashan2023MaskFreeOO}.

Recently, with the rapid advancement of Vision Foundation Model (VFM)~\cite{radford2021learning,kirillov2023segment,ravisam} pre-trained on massive scales of 2D images, \emph{training-free zero-shot} 3D scene understanding has emerged and been attracting significant attention from the research community~\cite{takmaz2023openmask3d,nguyen2024open3dis,yan2024maskclustering,hsu2025openm3d,Zhao_2026_CVPR,lemeshko2026zoo3d}. For instance, early methods~\cite{yang2023sam3d,Lu2023OVIR3DO3} leverage the powerful 2D segmentation model SAM~\cite{kirillov2023segment} to generate 2D instance masks which are subsequently lifted onto point clouds to achieve 3D perception. This line of 2D image segmentation-based methods enables class-agnostic 3D segmentation in a training-free manner, effectively circumventing the challenge of training robust models on scarce 3D scene data. Concurrently, other efforts~\cite{takmaz2023openmask3d,hsu2025openm3d} exploit the \emph{open-vocabulary} capabilities of the vision-language model CLIP~\cite{zhu2024survey,radford2021learning} to perform open-vocabulary 3D semantic segmentation and object detection. As a result, the primary challenge in 3D scene understanding has shifted toward projecting multi-view 2D images into a joint 3D space and effectively establishing robust instance-level merging strategies, inspiring the design of numerous novel VFM-based approaches~\cite{yin2024sai3d,guo2024sam,wang2025open,wang2025voxelflow,nguyen2025open,Zhao_2026_CVPR,lemeshko2026zoo3d}.

Despite their notable progress in 2D VFM-based 3D scene understanding, these methods still face two critical challenges. \\
1) The first challenge is \emph{streaming perception}. In practical applications, algorithms typically need to incrementally process frame-by-frame RGB-D inputs to generate dynamic 3D instance information for online usage, without prior access to the full scene~\cite{bogoslavskyi2016fast,grinvald2019volumetric,menini2021real}. However, most existing zero-shot methods operate in an offline, full-sequence setting, requiring the global scene and all historical frames to produce segmentation or detection results~\cite{yang2023sam3d,yan2024maskclustering,hsu2025openm3d,lemeshko2026zoo3d}, rendering them unsuitable for real-world sequentially-arriving data. Intuitively, streaming perception could be achieved by repeatedly running a full-sequence method on each newly arrived frame alongside historical frames. However, such approaches suffer from high computational complexity and fail to provide timely responses, as the scale of the reconstructed point cloud and the model inference cost grow rapidly with the number of frames. \\
2) The second challenge is \emph{robustness against noise masks}. Because zero-shot 3D perception relies on 2D models to generate segmentation masks that are then projected into 3D space, inherent task differences between 2D and 3D segmentation, coupled with mismatches between 2D masks and 3D depth, inevitably produce noise 3D masks~\cite{lei2023recent,Xu_2026_CVPR}. In streaming segmentation tasks, \eg these noise 3D masks come from limited views that cannot be refined from a global perspective. Consequently, inaccurate intermediate results cumulatively degrade the subsequent 3D perception process. While some efforts have explored 2D VFM-based streaming 3D segmentation~\cite{tie20242,yamazaki2024open,tang2025onlineanyseg,xuembodiedsam,wang2025online}, these methods either still require labeled data to train 3D models or fail to adequately mitigate the negative impact of noise masks when utilizing 2D VFMs.

In this paper, we propose Stream3Dv2 to directly address these limitations, as illustrated in Fig.~\ref{fig:pipeline}(b-d). Stream3Dv2 performs streaming 3D perception through a nested local to historical framework, providing robust segmentation results for downstream 3D scene understanding tasks. Specifically, when utilizing VFMs for 2D segmentation, we observe that masks generated by grid prompts~\cite{ravisam} are class-agnostic and exhibit high scene coverage, but are often coarse and noisy. Conversely, masks generated by semantic prompts~\cite{carion2026sam} possess precise boundaries and explicit class information, but typically suffer from low scene coverage and semantic ambiguity. Simply combining these two types of masks via intersection can lead to severe errors and fails to exploit their complementary strengths. To this end, we propose a novel local multi-view geometric-semantic fusion segmentation method. This approach focuses on achieving robustness against noise 3D masks and semantic ambiguity within a local multi-view window by formulating the 3D segmentation task as point-and-set merging and partitioning problems. Furthermore, we leverage the geometric information of point clouds to enable local 3D mask refinement and efficient global 3D mask updating. Concretely, we construct graph structures based on the manifold metrics of instances, utilizing adjacency relationships in a local non-Euclidean space to perform point-to-manifold mask refinement. The geometric bounding box of each 3D mask is then employed to dynamically activate static, dynamic, and newly detected instances within the historical mask pool, facilitating rapid manifold-to-manifold mask updates.

\noindent The primary contributions of this work are as follows:
\begin{itemize}
    \item We propose a novel streaming 3D perception framework (\ie Stream3Dv2), where the core instance-level segmentation task is decomposed into nested local multi-view and local-to-historical sub-tasks. This design enables us to effectively exploit multi-view consistency while maintaining time-efficient streaming perception.
    \item We introduce a comprehensive geometric-semantic fusion method that performs semantic-driven mask merging and partitioning over noisy VFM-based 2D segmentation. This combines mask denoising and fine-\&coarse-grained mask fusion, jointly ensuring robust 3D segmentation.
    \item We design an innovative point cloud refinement strategy. Our point-to-manifold optimization on a local manifold graph overcomes the limitations of Euclidean metrics, while the manifold-to-manifold update dynamically activates historical masks to achieve efficient optimization.
    \item Our framework is zero-shot, robust, and entirely training-free, making it highly scalable for open-world 3D perception. Extensive experiments on public datasets demonstrate that our Stream3Dv2 yields superior foundational zero-shot segmentation and detection performance, and attains advanced language-driven 3D scene understanding ability when it integrates with an LLM-based agent.
\end{itemize}

\textbf{Innovations over Stream3D}~\cite{Xu_2026_CVPR}:
Our preliminary conference work, Stream3D, decomposes streaming zero-shot 3D instance segmentation into a purely geometric local multi-view 3D segmentation phase and a Euclidean-distance-based local-to-historical merging phase. While this design facilitates noise mask suppression and incremental frame updates, it suffers from three fundamental limitations: 
(i) A lack of semantic guidance during local 3D mask merging and refinement. The class-agnostic segmentation and threshold-based merging approaches frequently cause over- and under-segmentation, leading to substantial deviations between an instance's true geometry and its segmented point cloud shape. 
(ii) Stream3D uses the Euclidean metrics to refine the point cloud manifolds, which is uneasy to accurately model the distance lying across non-Euclidean manifold shapes of point clouds, and thus might result in unsatisfactory instance boundary delineation. Such errors are difficult to be corrected and will accumulate progressively during subsequent merging and updating stages.
(iii) A reliance (common among many VFM-based methods including Stream3D) on high-frequency comparisons and similarity computations of soft semantic features in high-dimensional spaces. This requires manually tuned similarity thresholds and struggles to resolve semantic ambiguity effectively.
In contrast, our proposed Stream3Dv2 is manifold-distance-based and explicitly codes hard semantics into a geometric-semantic fusion process, successfully eliminating geometric noise and semantic ambiguity with minimal computational overhead. As demonstrated in Sec.~\ref{Experiments}, these architectural enhancements constitute a fundamental methodological upgrade, allowing Stream3Dv2 to consistently outperform existing baselines. Furthermore, we extend Stream3Dv2 to a broader range of 3D scene understanding tasks, validating its practical applicability.

\section{Related Work}\label{sec:relatedworks}

\subsection{3D Scene Understanding}
3D scene understanding is an attention-getting research area that aims to infer both geometric structures and semantic attributes on 3D scene data (\eg point clouds~\cite{lomenie2004generic}, mesh~\cite{jagannathan2007three}, and volumetric grids~\cite{grinvald2019volumetric}).
Early methods heavily rely on hand-crafted geometric features combined with traditional machine learning techniques like point clustering and graph node optimization~\cite{reitberger20093d,douillard2011segmentation}.
Driven by deep learning, numerous works have integrated fast-growing neural network models to extract 3D scene features, remarkably enhancing spatial understanding capability.
The popular neural network architectures include point-based~\cite{qi2017pointnet}, voxel-based~\cite{tchapmi2017segcloud}, transformer-based~\cite{zhao2021point}, and multimodal models~\cite{hou20193d}.
They are trained by learning objectives to support specific 3D scene understanding tasks including classification, segmentation, detection, \etc

However, previous 3D scene understanding methods are typically based on supervised learning which relies on dense manual annotations and closed-set categorical assumptions. 
Consequently, they suffer from poor generalization performance when encountering unseen classes or new scenes~\cite{yan2024maskclustering,hsu2025openm3d}. 

\subsection{Open-Vocabulary Zero-Shot 3D Perception and Interaction}
Fortunately, achieving zero-shot 3D scene understanding has become possible in recent years, due to the progress of vision foundation models~\cite{radford2021learning,kirillov2023segment}.
To be specific, recent approaches~\cite{yang2023sam3d,yan2024maskclustering,lemeshko2026zoo3d} leverage the off-the-shelf 2D image segmentation models (\eg SAM~\cite{kirillov2023segment} and Cropformer~\cite{qi2023high}) to facilitate zero-shot 3D instance segmentation and object detection.
Meanwhile, the open-vocabulary query approach in CLIP~\cite{radford2021learning} enables existing 3D perception methods to achieve open-vocabulary semantic annotation and no longer rely on training a fixed classification network.
This demonstrates the potential of transferring semantic knowledge from 2D to 3D and inspires more and more methods in embodied AI~\cite{nguyen2024open3dis,yin2024sai3d,guo2024sam}.
Such open-vocabulary zero-shot 3D framework serves not merely as a perception task, but as a foundational primitive that bridges unstructured 3D geometry with human-level embodied interaction, where converting point clouds into object-centric representation units is necessary for language-directed manipulation, visual grounding, and spatial reasoning~\cite{wang2026affordbot,jiang2026taven}.

Despite these advances, critical bottlenecks remain unresolved.
First, cross-modal discrepancies between 2D vision and 3D geometry, coupled with sensor noise and semantic differences across views, frequently induce noise 3D masks with geometric inaccuracies.
Second, most zero-shot 3D perception methods rely heavily on offline global scene information, creating a significant barrier to real-time online deployment where agents need to process streaming RGB-D inputs. 

\subsection{Streaming 3D Instance Segmentation}
It is essential for real-world applications that 3D segmentation algorithms should incrementally process incoming RGB-D frames and produce near-real-time dynamic segmentation results~\cite{tateno2015real,mccormac2018fusion}.
To this end, many works~\cite{narita2019panopticfusion,liu2022ins} leverage the well-trained point cloud models to extract point cloud features and design fast feature aggregation methods for online 3D instance segmentation.
These methods can produce timely responses in streaming 3D instance segmentation tasks, but their performance heavily relies on the models trained on specific datasets, and the good performance is difficult to generalize to unknown instances and complex scenes~\cite{ngo2023isbnet,weder2025alster}.

To employ the powerful ability of 2D vision foundation models while considering fast processing, recent pipelines~\cite{tie20242,yamazaki2024open,xuembodiedsam} train lightweight networks built on 2D segmentation results to achieve streaming 3D mask merging. Motivated by zero-shot 3D segmentation~\cite{yang2023sam3d,yan2024maskclustering}, some methods~\cite{tang2025onlineanyseg,du2026moonseg3r} are further proposed to provide training-free streaming 3D segmentation. Despite their significant progress, prior VFM-based methods either still need labeled data for training in streaming 3D segmentation or do not fully consider to address the negative effects of noise masks in training-free manner.

\section{Method}
\subsection{Stream3Dv2: Geometric-Semantic Fusion Framework for Open-Vocabulary Streaming Zero-Shot 3D Segmentation}

We begin by formulating the full-sequence zero-shot 3D instance segmentation (FZ3DS) and then define its streaming counterpart, followed by our alternative paradigm.

\textbf{Preliminary.}
In the FZ3DS paradigm, the input consists of a sequence of RGB images $\{I^t\in \mathbb{R}^{H\times W\times 3}\}_{t=1}^T$ along with a globally reconstructed point cloud $P\in \mathbb{R}^{N\times 3}$ obtained from the corresponding depth maps and camera poses across all $T$ frames. A pre-trained 2D segmentation model (\eg Cropformer \cite{qi2023high} or SAM series~\cite{kirillov2023segment,ravisam}) is applied to each image to produce 2D masks which are subsequently back-projected into the 3D space. This yields a collection of per-frame 3D instance masks $\{M^t = \{m_s^t\}_{s=1}^{s_t}\}_{t=1}^T$, where $m_s^t$ denotes the index set of points in $P$ assigned to the $s$-th mask at frame $t$. Existing approaches~\cite{yang2023sam3d,Lu2023OVIR3DO3,yan2024maskclustering} then merge these partial masks across frames to form the final class-agnostic instance segmentation masks, denoted as $M^{1:T}$, and the FZ3DS pipeline is:
\begin{equation}\label{FZ3DS}
    M^{1:T} = \mathcal{F}_{\text{FZ3DS}}\left(\{I^t,M^t\}_{t=1}^T;P\right).
\end{equation}
Then, the entire FZ3DS leverages a vision-language model like CLIP \cite{radford2021learning} for open-vocabulary semantic labeling.

By contrast, streaming zero-shot 3D instance segmentation (SZ3DS) imposes a more stringent condition, \ie at each frame $t$, the model has access only to the current RGB-D pair and must incrementally update the segmentation results for the partially observed scene in a frame-by-frame fashion. The process is defined recursively as follows:
\begin{equation}\label{SZ3DS}
    M^{1:t} = \mathcal{F}_{\text{SZ3DS}}\left(I^t;M^t;P^t \mid M^{1:t-1};P^{1:t-1}\right),
\end{equation}
where $P^{1:t-1}$ represents the accumulated point clouds from the first $t-1$ frames, $P^t$ is the local point clouds reconstructed from the $t$-th frame, $M^{1:t-1}$ are the historical masks from previous frames, and $M^{1:t}$ denotes the dynamically updated masks after incorporating frame $t$. As $t$ increases, the 3D scene segmentation grows incrementally. Compared to the full-sequence formulation in Eq.~(\ref{FZ3DS}), Eq.~(\ref{SZ3DS}) captures a dynamic, sequential setting, yet it introduces two major challenges, including the sensitivity to noise masks and high computational latency, as discussed in Sec.~\ref{sec:intro}.

\textbf{Stream3Dv2 framework.}
To address these challenges in SZ3DS, we propose a novel Stream3Dv2 paradigm which decomposes $\mathcal{F}_{\text{SZ3DS}}$ into two nested sub-tasks, including local multi-view 3D segmentation $\mathcal{F}_{\text{Loc}}$ and local-to-historical 3D segmentation $\mathcal{F}_{\text{Loc2his}}$. The paradigm is formulated as follows:
\begin{equation}\label{eq:stream3dpp}
\begin{aligned}
    & M^{1:t} = \mathcal{F}_{\text{Loc2his}}\Bigl(M^{(t-k+1):t} \mid M^{1:t-k}; P^{1:t}\Bigr) \\
    & \text{s.t.}~M^{(t-k+1):t} = \mathcal{F}_{\text{Loc}}\Bigl(\left\{I^i;G^i;[S^i,L^i];P^i\right\}_{i=t-k+1}^{t}\Bigr),
\end{aligned}
\end{equation}
where \(k\) is the sliding window size retaining the most recent \(k\) frames.
For each frame \(i\), \(G^i\) denotes the class‑agnostic 3D masks of all instances obtained by grid-prompted segmentation with vision foundation models~\cite{ravisam},
and \(S^i\) is the semantic 3D masks obtained by semantic-prompted segmentation with text prompts \(L^i\) from common scene classes~\cite{carion2026sam}.
There is a one-to-one correspondence between the elements of set \(S^i\) and set \(L^i\), \eg a 3D point cloud mask \(m_s^i \in S^i\) corresponds to a textual word \(desk \in L^i\): \([m_s^i,desk] \in [S^i,L^i]\).
$\mathcal{F}_{\text{Loc}}$ achieves the robust 3D instance segmentation for the local scene, and then $\mathcal{F}_{\text{Loc2his}}$ merges the results into the seen historical scene.
In this way, our Stream3Dv2 paradigm can not only explore the geometric-semantic complementarity and multi-view consistency to improve robustness against noise 3D masks, but also avoid involving the irrelevant views in the scene thus to meet the timely response for SZ3DS.

In the next subsections, we will introduce our specific methodology for the local multi-view 3D segmentation $\mathcal{F}_{\text{Loc}}$ (in Sec.~\ref{mv},\ref{fg}, and \ref{mf}) as well as the local-to-historical streaming 3D segmentation $\mathcal{F}_{\text{Loc2his}}$ (in Sec.~\ref{l2h}).

\subsection{Local Multi-View Coarse-Grained 3D Segmentation}\label{mv}
\begin{figure}[t]
\centering
\includegraphics[width=1\linewidth]{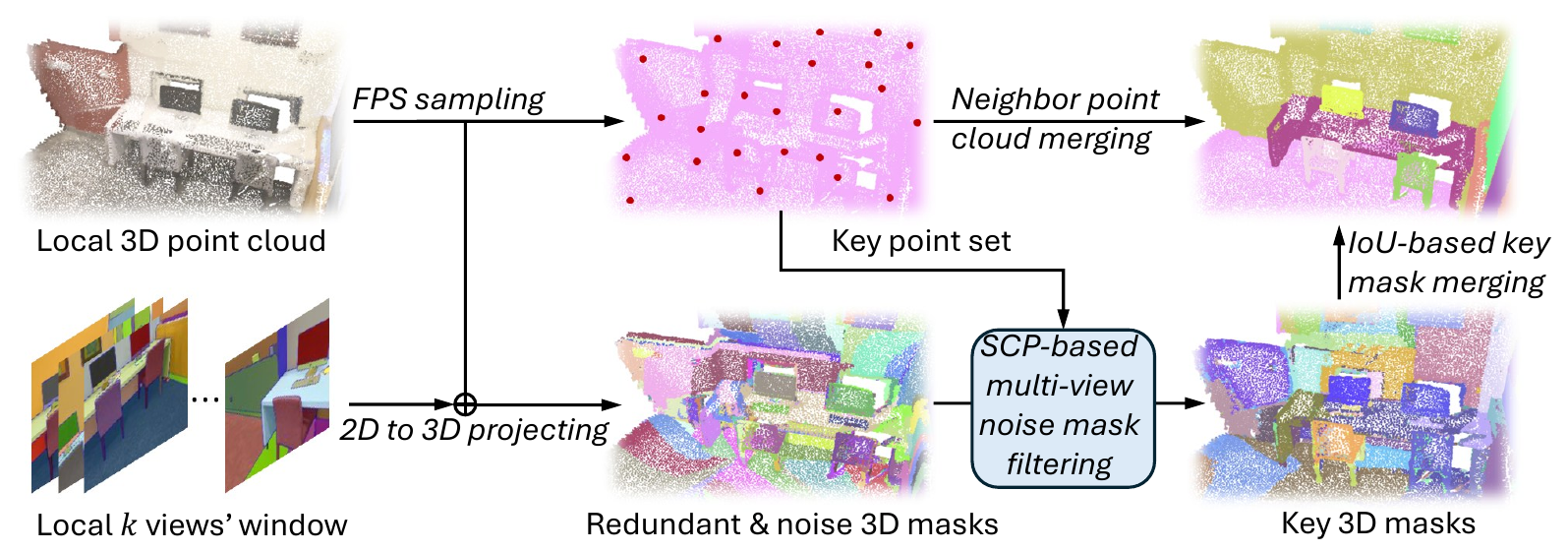}
\caption{
Local multi-view coarse-grained 3D segmentation.
}\label{fig:scp}
\end{figure}
Grid-prompted segmentation~\cite{ravisam} on images with vision foundation models is conducive to rapid and thorough 2D segmentation for possible instances, \ie the instances in each frame are fully covered.
However, the obtained results are class-agnostic and often accompanied by over-segmentation masks.
To effectively utilize the advantages of grid-prompted segmentation, we propose the following \emph{Geometry-based single-view mask denoising}, \emph{SCP-based multi-view noise mask filtering}, and \emph{IoU-based key mask merging}, which focus on obtaining preliminary coarse-grained 3D segmentation from noise and over-segmentation masks as shown in Fig.~\ref{fig:scp}.

\textbf{Geometry-based single-view mask denoising.}
By intuition, the point cloud of a single instance should form a contiguous manifold, while the boundaries between different instances ought to be clearly separated to prevent merging errors.
Inspired by intra-cluster compactness and inter-cluster separability~\cite{souvenir2005manifold,Xu_2025_ICCV}, we regard each scene instance as a cluster in the 3D space and enforce two distance constraints to purify every 3D mask.
Formally, letting ${m}_a^t,{m}_b^t \in G^i$ be the point index sets of any two masks in the $t$-th frame, we impose
\begin{equation}\label{dis}
\begin{aligned}
\begin{cases}
    \forall i \in {m}_a^t,\exists j \in {m}_a^t&\vert~~~ dis(p_{(i)},p_{(j)}) < \delta, \\
    \forall i \in {m}_a^t,\forall j \in {m}_b^t&\vert~~~ dis(p_{(i)},p_{(j)}) > \delta, 
\end{cases}
\end{aligned}
\end{equation}
where $p_{(i)}, p_{(j)} \in \mathbb{R}^3$ are the coordinates of points $i$ and $j$, $dis(\cdot,\cdot)$ denotes the Euclidean distance, and $\delta$ is a threshold defined in the point cloud space.
In Eq.~(\ref{dis}), the first condition enforces \emph{intra-instance continuity}, ensuring that each mask corresponds to a spatially coherent manifold and does not cover points from non-adjacent instances. 
The second condition imposes \emph{inter-instance separation}, which suppresses boundary noise in the point cloud stemming from projection inaccuracies.
Points that violate the above constraints are discarded to conduct mask denoising within each view.

\textbf{SCP-based multi-view noise mask filtering.}
We make a basic observation that as the number of local frames increases, more 2D masks are generated by the 2D segmentation model, which in turn results in more 3D masks. However, these 3D masks are often redundant, incomplete, or noisy, which can significantly degrade the segmentation performance.
To address this, we first define key points and then design a set-covering-problem based multi-view noise mask filtering strategy, figuring out key masks among local frames for noise-robust local multi-view 3D instance segmentation.

Specifically, for $k$ frames observed from multiple views of a local scene, the corresponding point cloud is $P^{(t-k+1):t}=\bigcup_{i=(t-k+1)}^tP^i \in \mathbb{R}^{N_{t} \times 3}$.
We leverage the farthest point sampling (FPS)~\cite{eldar1997farthest} with a key point rate $\gamma=0.05$ on $P^{(t-k+1):t}$ to obtain the key points $\mathcal{O}^{(t-k+1):t} \in \mathbb{R}^{n_t \times 3}, n_t < N_{t}$,
and let $\mathcal{K}^{(t-k+1):t}$ denote the set of indices for these key points.
These key points are uniformly distributed in the local point cloud and help us to define the key masks. 
Next, we define $\mathcal{C}^{(t-k+1):t}$ as the union set of all 3D class-agnostic masks across $k$ frames, \ie $\mathcal{C}^{(t-k+1):t} = \bigcup_{i=t-k+1}^t G^i = \bigcup_{i=t-k+1}^t\{m_1^i, m_2^i, \dots, m_{s_i}^i\}$,
and then solve a set covering problem to select the key mask set $\mathcal{M}$ as follows:
\begin{equation}\label{scp}
\begin{aligned}
&\arg\min_{\mathcal{M} \subseteq \mathcal{C}^{(t-k+1):t}} |\mathcal{M}| \\ & s.t.~\forall p \in \mathcal{K}^{(t-k+1):t}, \exists m \in \mathcal{M}~\vert~p \in m,
\end{aligned}
\end{equation}
where $m$ denotes a key mask in $\mathcal{M}$ and $p$ is the index of a key point from $\mathcal{K}^{(t-k+1):t}$.
The objective of Eq.~(\ref{scp}) is to select the smallest number of 3D segmentation masks in the local scene such that these key masks fully cover the key points.
This is conducive to filtering out noise masks and reducing the chances of interference across multiple views.

\textbf{IoU-based key mask merging.}
Given the set of key masks $\mathcal{M}$, they share the consistency and complementarity among multiple views of the same instance, which motivates us to conduct key mask merging to obtain coarse-grained masks.

Specifically, we use ${A} \in \{0,1\}^{n \times n}, n = |\mathcal{M}|$ to represent the connectivity matrix among all key masks.
For any two key masks $m_i, m_j \in \mathcal{M}$, they should be connected (\ie ${A}_{ij} = 1$) if their Intersection over Union (IoU) exceeds a threshold (\eg overlap threshold $\alpha = 0.2$), otherwise ${A}_{ij} = 0$.
This point-wise intersection consistency allows us to identify masks belonging to the same instance.
The number of connected components in ${A}$ implies the number of instances that should exist in the local $t$ frames.
The point union complementarity of multiple masks from each connected component can assist us to obtain complete 3D masks.
We define the reachability matrix ${R} = {A} \lor {A}^2 \lor \dots \lor{A}^n$ and merge key masks as follows:
\begin{equation}
    m_i = \bigcup_{j} m_j \quad s.t.~m_i, m_j \in \mathcal{M},{R}_{ij} = 1. 
\end{equation}
Then, for the local window's $t$ frames, the merged key masks are the coarse-grained, class-agnostic 3D masks $G^{(t-k+1):t}$.

\subsection{Semantic-Driven Fine-Grained 3D Segmentation}\label{fg}
Meanwhile, for the local $k$ frames, we conduct the semantic-prompted segmentation~\cite{carion2026sam} on each image, using the common scene classes as prompts to obtain semantic masks $\{S^i\}_{i=t-k+1}^{t}$.
Since each point in semantic masks has the class label, we directly count the most frequent labels for each point to mitigate semantic ambiguity, achieving the de-duplication of semantic masks and voting out 3D semantic segmentation masks $S^{(t-k+1):t}$.
Semantic-prompted segmentation is good at obtaining masks with explicit class information, but is limited by the prompt words and thus has difficulty achieving full-instance segmentation.
To combine the advantages of semantic- and grid-prompted segmentation, we propose \emph{Set-Partitioning-Problem (SPP) based semantic-driven mask partitioning}, which leverages semantic 3D masks $S^{(t-k+1):t}$ to further refine the ready coarse-grained 3D masks $G^{(t-k+1):t}$ to obtain accurate fine-grained segmentation, as shown in Fig.~\ref{fig:sdm}.

Concretely, denote the local window frames as $\Delta=(t-k+1):t$, the coarse-grained masks are ${G}^\Delta=\{m_g^\Delta\}_{g=1}^{g_\Delta}$, and the semantic masks are ${S}^\Delta=\{m_s^\Delta\}_{s=1}^{s_\Delta}$.
To ensure that the fine-grained semantic masks are not covered by themselves, we prioritize retaining the smaller masks and remove duplicate points from the larger ones, resulting in mutually disjoint semantic masks, \ie $m_s^\Delta\cap m_{s'}^\Delta=\emptyset$ for $s\neq s'$.
Then, for each semantic mask $m_s^\Delta$ and coarse-grained mask $m_g^\Delta$, we adopt the same concept of mask overlapping to define the inclusion degree as $\rho_{s,g}=|m_s^\Delta\cap m_g^\Delta|\,/\,|m_s^\Delta|$.
A semantic mask should be fused to the coarse-grained mask with the largest inclusion degree, controlled by the overlap threshold $\alpha\in[0,1]$:
$g^*(s)=\arg\max_g\rho_{s,g}$ subject to $\rho_{s,g^*(s)}\ge1-\alpha$.
$\alpha=0$ requires perfect containment, whereas $\alpha<1$ tolerates boundary leakage, essential for real-world noise masks.
% $\alpha'=1-\alpha$ acts with the similar as the previous overlap threshold $\alpha$.
Moreover, for every fused semantic mask we form its cropped point set $c_s=m_s^\Delta\cap m_{g^*(s)}^\Delta$.
A coarse-grained mask $m_g^\Delta$ might have multiple cropped parts with different semantic masks.
The cropped parts associated with $m_g^\Delta$ are collected in $\mathcal{C}_g=\{c_s\mid g^*(s)=g\}$, which are pairwise disjoint subsets of $m_g^\Delta$ due to the disjoint semantic masks.
This implies that a coarse-grained mask can be decomposed into multiple semantic masks, and we let $S_g=\sum_{c_s\in\mathcal{C}_g}|c_s|$ denote the total point count of decomposed parts.

\begin{figure}[t]
\centering
\includegraphics[width=1\linewidth]{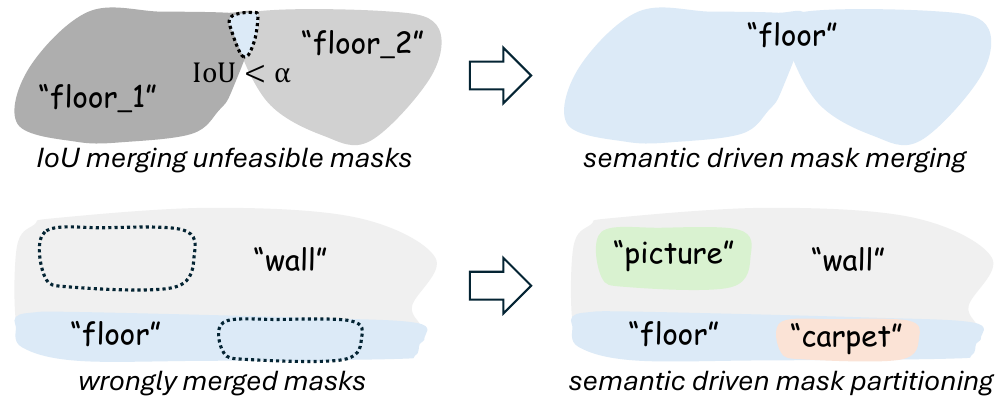}
\caption{3D mask merging and partitioning. \emph{Left:} threshold-based methods suffer from dynamical overlapping or wrongly-merged masks. \emph{Right:} semantic-driven method (our work) is able to make more feasible mask merging and partitioning with semantic information.}\label{fig:sdm}
\end{figure}
Furthermore, we introduce binary variables $\mathbf{z}\in\{0,1\}^{g_\Delta}$ and $\mathbf{w}\in\{0,1\}^{s_\Delta}$, where $z_g=1$ indicates that the coarse-grained mask $m_g^\Delta$ is replaced by the fine parts $\mathcal{C}_g$, while $w_s=1$ retains the original semantic mask $m_s^\Delta$ as a new instance.
In this way, we formulate the semantic-driven mask partitioning as maximizing the following objective $\mathcal{J}$:
\begin{equation}
\begin{aligned}
\max_{\mathbf{z},\mathbf{w}} \mathcal{J}& = \sum_{g=1}^{g_\Delta} \big[ z_g S_g + (1-z_g)|m_g^\Delta| + \lambda z_g |\mathcal{C}_g| \big] + \sum_{s=1}^{s_\Delta} w_s |m_s^\Delta|, \label{decomposition} \\ 
\text{s.t.}\quad &w_s + \sum\nolimits_{g:g^*(s)=g} z_g \le 1,\;\forall s, \\ 
                 &S_g z_g \ge |m_g^\Delta| z_g,\;\forall g, \\ 
                 % &S_g z_g \ge (1-\tau)|m_g^\Delta| z_g,\;\forall g, \\ 
                 &w_s \le \mathbbm{1}\!\big[ m_s^\Delta \cap (\cup_g m_g^\Delta) = \emptyset \big],\;\forall s. 
\end{aligned}
\end{equation}
In Eq.~(\ref{decomposition}), if a coarse-grained mask is decomposed (\eg $z_g=1$), the gain is $S_g$ (\ie the points of cropped parts covered by semantic masks) plus a decomposition-encouraging bonus $\lambda|\mathcal{C}_g|$; otherwise ($z_g=0$), the original $|m_g^\Delta|$ points of the mask are kept.
In the first term, $\lambda$ encourages coarse-grained mask decomposition when semantic masks meet the condition (we set it with a sufficiently large value in experiments).
The second term adds points of semantic masks that become novel instances ($w_s=1$).
The first constraint ensures each semantic mask is used at most once (either in a decomposition or as a novel instance).
The second constraint enforces that any point loss in a decomposed coarse-grained mask is bounded by $|m_g^\Delta|$.
The third constraint guarantees novel instances are completely disjoint from all coarse-grained masks.

After solving the SPP problem, the refined mask set is
\begin{equation}
M^{\Delta} =  \bigcup_{g:z_g=1} \mathcal{C}_g \cup \big\{ m_g^\Delta \mid z_g=0 \big\}\cup \big\{ m_s^\Delta \mid w_s=1 \big\},
\label{eq:D}
\end{equation}
where $\Delta=(t-k+1):t$. As a result, we not only integrate semantic masks to produce more informative 3D instance segmentation, but also can reasonably split the class-agnostic coarse-grained masks based on the fine-grained masks.

\subsection{Point Cloud Manifold Refinement}\label{mf}

The geometric constraint in Eq.~(\ref{dis}), coverage condition in Eq.~(\ref{scp}), and decomposition operation in Eq.~(\ref{decomposition}) primarily take the manner of discarding ambiguous points to conservatively optimize point clouds.
For a local scene, let $\mathcal{U}^{\Delta}$ be the set of points not yet assigned to any mask, and allocating these points to the masks in $M^{\Delta}$ obtains the complete 3D instance segmentation.
Euclidean distance, however, cannot capture the non‑Euclidean surface geometry, leading to incorrect assignments across curved or thin structures.
To this end, we introduce point‑to‑manifold assignment that builds on manifold-aware graph metrics and solves the multi‑source Eikonal equation~\cite{sethian1999level}, motivated by a wavefront propagation model that computes geodesic distances on a curved surface, to achieve physically meaningful manifold refinement.

\textbf{Manifold-aware graph metric.}
Formally, the local scene contains $n^{\Delta}$ points, each identified by an index $i \in \mathcal{I}^{\Delta}$ and a coordinate $p_{(i)} \in \mathbb{R}^3$.
Every mask $m_c \in M^{\Delta}$ is a subset of $\mathcal{I}^{\Delta}$.
Given $M^{\Delta} = \{m_c\}_{c=1}^{C^{\Delta}}$ and the unassigned point set $\mathcal{U}^{\Delta} = \mathcal{I}^{\Delta} \setminus \bigcup_{c=1}^{C^{\Delta}} m_c$,
the goal is to assign each $i \in \mathcal{U}^{\Delta}$ to its manifold‑closest mask in $M^{\Delta}$.
To efficiently represent the surface of point cloud manifold,
we construct a graph $\mathcal{G}=(\mathcal{V},\mathcal{E})$ where node set $\mathcal{V}$ collects the super‑points of grid voxelization in 3D point cloud,
and edge set $\mathcal{E}$ models the accessibility degree between super‑points.
Specifically, each node $v \in \mathcal{V}$ corresponds to a subset $\mathcal{I}_v^{\Delta} \subseteq \mathcal{I}^{\Delta}$, characterized by its centroid $\mathbf{c}_v$ and regularized covariance $\Sigma_v$:
\begin{equation}\label{eq:centcov}
\begin{aligned}
\mathbf{c}_v &= \frac{1}{|\mathcal{I}_v^{\Delta}|} \sum_{i \in \mathcal{I}_v^{\Delta}} p_{(i)}, \\
\Sigma_v     &= \frac{1}{|\mathcal{I}_v^{\Delta}|} \sum_{i \in \mathcal{I}_v^{\Delta}} \bigl(p_{(i)}-\mathbf{c}_v\bigr)\bigl(p_{(i)}-\mathbf{c}_v\bigr)^{\!\top} + \varepsilon \mathbf{I},
\end{aligned}
\end{equation}
where $\varepsilon=10^{-6}$ ensures positive definiteness even for near‑planar or degenerate shapes, thus yielding the Riemannian metric $G_v = \Sigma_v^{-1}$ (when $|\mathcal{I}_v^{\Delta}| < 4$, $G_v = \mathbf{I}$ avoids unreliable covariance estimates).
In graph $\mathcal{G}$, an edge $(u,v)$ is inserted if $\lVert\mathbf{c}_u - \mathbf{c}_v\rVert < \sqrt{3} {(\frac{n_e}{\rho})}^{1/3}$ where $n_e$ is the expected average number of points per pixel (\eg 20) and $\rho$ is the global average density automatically calculated from the point cloud, thereby to capture the local irregular manifold shapes.
The weight of edge is the symmetric Mahalanobis distance:
\begin{equation}
d_{\mathcal{M}}(u,v) = \sqrt{ (\mathbf{c}_u - \mathbf{c}_v)^{\!\top} \left(\frac{G_u + G_v}{2}\right) (\mathbf{c}_u - \mathbf{c}_v) }.
\label{eq:mah}
\end{equation}
Since $G_v$ assigns small weights to tangential displacements and large weights to normal displacements, $d_{\mathcal{M}}$ strongly penalizes off‑surface movement while keeping tangential motion cheap. The shortest paths in $\mathcal{G}$ therefore follow the intrinsic surface geometry, thereby avoiding the failure of Euclidean metric on non-Euclidean point cloud manifold.

\textbf{Point-to-manifold mask optimization.}
The anisotropic Eikonal equation~\cite{sethian1999level} describes a propagating front whose speed depends on the local surface orientation, and its solution gives the minimal travel time by geodesic distance from a set of source points.
For convenience, we treat each mask $m_c$ as a class label $c$.
Then, we instantiate one Eikonal equation for per mask $m_c$ and define a geodesic distance field $\phi_c:\mathcal{V}\to\mathbb{R}^+\cup\{\infty\}$ that records, for every node, the minimal cost along the graph from the sources of class $c$.
A node $v$ acts as a seed (the source where the front of class $c$ originates with zero cost) if its super‑point contains at least one point already labelled as $c$, \ie $\mathcal{I}_v^{\Delta} \cap m_c \neq \emptyset$.
For each $c$, $v \in \mathcal{V}$, $\phi_c$ satisfies the discrete Eikonal equation
\begin{equation}\label{eq:eikonal}
\phi_c(v) = 
\begin{cases}
0, & \mathcal{I}_v^{\Delta} \cap m_c \neq \emptyset,\\[6pt]
\displaystyle\min_{u \in \mathcal{N}(v)} \bigl\{ \phi_c(u) + d_{\mathcal{M}}(v,u) \bigr\}, & \text{otherwise},
\end{cases}
\end{equation}
where $\mathcal{N}(v) = \{u \in \mathcal{V} \mid (u,v) \in \mathcal{E}\}$ are the immediate graph neighbors of $v$.
This formulation casts the point‑to‑manifold assignment as a competition among $C$ geodesic distance fields, \ie each field emanates from its own seeds, and a node is claimed by the mask whose front reaches it first.

\textbf{Iterative Bellman relaxation.}
To find the shortest-path distance of point-to-manifold assignment on the manifold-aware graph, we employ an iterative Bellman relaxation scheme~\cite{golub2013matrix} to solve our point-to-manifold mask optimization.
Specifically, we evolve all $C$ coupled fields simultaneously with multiple iterations, and the iterative solutions progressively propagate the minimum accumulated cost over the super-point graph $\mathcal{G}$.
we set $R\in \mathbb{N}^+$ iterations to make distance information propagate across all connected components on $\mathcal{G}$, encouraging that the minimum-cost information could reach every node within the manifold-aware graph structures.
Then, we initialize $\phi_c^{(0)}(v) = 0$ for seeds of $m_c$ and $+\infty$ otherwise. For $r = 1, \dots, R$, update non‑seed nodes by
\begin{equation}\label{eq:sweep}
\phi_c^{(r)}(v) \leftarrow
\min\Bigl\{\phi_c^{(r-1)}(v),\;
\min_{u \in \mathcal{N}(v)} \bigl\{ \phi_c^{(r-1)}(u) + d_{\mathcal{M}}(v,u) \bigr\}\Bigr\}.
\end{equation}
Node labels are then assigned as
\begin{equation}\label{eq:node_label}
\ell(v) = 
\begin{cases}
\arg\min\limits_{c} \phi_c^{(R)}(v), & \min_c \phi_c^{(R)}(v) < \infty,\\
-1, & \text{otherwise}.
\end{cases}
\end{equation}

After the $R$ iterations, each unassigned point $i \in \mathcal{U}^{\Delta}$ is mapped to its enclosing super‑point $v_i$ (\ie the unique node with $i \in \mathcal{I}_{v_i}^{\Delta}$).
The point inherits the label $\ell(v_i)$ if $\ell(v_i) \neq -1$; otherwise it remains unassigned.
The refined local segmentation mask set is then updated as follows:
\begin{equation}
M^{\Delta} \leftarrow \bigl\{\, m_c \cup \{\, i \in \mathcal{U}^{\Delta} \mid \ell(v_i) = c \,\} \,\bigr\}_{c=1}^{C^{\Delta}} .
\label{eq:refined_mask}
\end{equation}

\subsection{Local-to-Historical Streaming 3D Instance Segmentation}\label{l2h}
Sec.~\ref{mv},~\ref{fg}, and~\ref{mf} established a complete pipeline for the local sub-task $\mathcal{F}_{\text{Loc}}$, \ie joint 3D instance and semantic segmentation within a single scene window.  
We now address the temporal propagation sub-task $\mathcal{F}_{\text{Loc2his}}$ in Eq.~(\ref{eq:stream3dpp}), which achieves manifold-to-manifold update from the freshly extracted local masks $M^{(t-k+1):t}$ (\ie the obtained $M^{\Delta}$) to the historical masks $M^{1:t-k}$ for producing the mask set $M^{1:t}$.

\textbf{3D detection-based fast instance localization.}
To realize streaming 3D instance segmentation, our paradigm iteratively updates the historical mask $M^{1:t-k}$ using the local masks $M^{(t-k+1):t}$ (\ie $M^{\Delta}$).
For computational efficiency, we compute bounding boxes for 3D detection and use them to quickly retrieve the instances that are spatially overlapping, and then determine the masks that need to be updated.

To be specific, for each mask $m$, let $\mathcal{P}(m) = \{ p_i = (x_i, y_i, z_i) \mid i \in m\}$ be the set of 3D points covered by $m$.
Its axis‑aligned bounding box (AABB) is defined by the minimum and maximum coordinate vectors as follows:
\begin{equation}
\begin{aligned}
    & \mathbf{b}(m) = \bigl(\mathbf{p}^{\min}_m,\; \mathbf{p}^{\max}_m\bigr),\\
    & s.t.~\mathbf{p}^{\min}_m = (\min_{p\in\mathcal{P}(m)} x,\; \min_{p\in\mathcal{P}(m)} y,\; \min_{p\in\mathcal{P}(m)} z),\\
    &~~~~~\mathbf{p}^{\max}_m = (\max_{p\in\mathcal{P}(m)} x,\; \max_{p\in\mathcal{P}(m)} y,\; \max_{p\in\mathcal{P}(m)} z).
\end{aligned}
\label{eq:aabb}
\end{equation}
Two masks have possibly spatial consistency only if their AABBs intersect.
Therefore, for a local mask $m_l^{(t-k+1):t}$, $\mathbf{b}(m_l)$, the candidate pool from the historical set $M^{1:t-k}$ is
\begin{equation}
B^{1:t-k}_{(l)} = \bigl\{m_j^{1:t-k} \in M^{1:t-k} \mid \mathbf{p}^{\min}_{m_l} \le \mathbf{p}^{\max}_{m_j};\ \mathbf{p}^{\min}_{m_j} \le \mathbf{p}^{\max}_{m_l}\bigr\}.
\label{eq:candidate_pool}
\end{equation}
In practice, the retrieval is accelerated to constant‑time lookups by pre‑indexing all historical AABBs in a spatial grid.
Since two masks cannot share any point unless their AABBs intersect, subsequent precise point‑level intersection operations for $m_l^{(t-k+1):t}$ are restricted exclusively to masks in $B^{1:t-k}_{(l)}$, while all other historical masks are safely skipped.

\textbf{Local-to-historical instance division and update.}
We define static, dynamic, and newly detected instances for manifold-to-manifold mask update.
Concretely, for each local mask $m_l^{(t-k+1):t} \in M^{(t-k+1):t}$, we split the historical masks in $B^{1:t-k}_{(l)}$ into \emph{static instances} $\mathcal{S}_{l}^{1:t-k}$ and \emph{dynamic instances} $\mathcal{D}_{l}^{1:t-k}$.
A historical mask $m_i^{1:t-k} \in B^{1:t-k}_{(l)}$ is classified as a dynamic instance if it shares at least one point with the local mask; otherwise, it is treated as static instance as follows:
\begin{equation}
\begin{cases}
m_i^{1:t-k} \in \mathcal{D}_l^{1:t-k}, & \text{if } m_l^{(t-k+1):t} \cap m_i^{1:t-k} \neq \emptyset, \\
m_i^{1:t-k} \in \mathcal{S}_l^{1:t-k}, & \text{otherwise}.
\end{cases}
\end{equation}
If $\mathcal{D}_l^{1:t-k} = \emptyset$, the local mask is treated as a \textit{newly detected instance}, \ie $m_l^{(t-k+1):t} \in \mathcal{N}^{(t-k+1):t}$.
For a local mask $m_l^{(t-k+1):t}$ and its associated dynamic set $\mathcal{D}_l^{1:t-k}$, the update rules are formulated as follows:
\begin{equation}\label{update}
\begin{cases}
m_i^{1:t} = m_i^{1:t-k} \cup m_l^{(t-k+1):t}, & \text{if } i = \arg\max_{j}|E_{lj}|, \\
m_i^{1:t} = m_i^{1:t-k} \setminus m_l^{(t-k+1):t}, & \text{else},
\end{cases}
\end{equation}
where $E_{lj} = m_j^{1:t-k} \cap m_l^{(t-k+1):t}$ is the intersection of the two masks and $m_i^{1:t} \in \mathcal{D}_l^{1:t}$ represents the updated mask.

In Eq.~(\ref{update}), the first case merges the local mask into the dynamic instance that exhibits the largest overlap.
The second case guarantees that the update strategy preserves the inter-instance separation constraint formulated in Eq.~(\ref{dis}).
Ultimately, the local-to-historical streaming segmentation result is the union of static, dynamic, and newly detected instances:
\begin{equation}\label{M}
M^{1:t} = \bigcup_{l} \bigl( \mathcal{S}_l^{1:t-k} \cup \mathcal{D}_l^{1:t} \bigr) \cup \mathcal{N}^{(t-k+1):t}.
\end{equation}
In this way, we not only efficiently handle dynamic and streaming scene understanding, but also make effective use of local-view consistency to facilitate robust perception.

\section{Experiments}\label{Experiments}

\begin{table*}[!t]
\caption{\textbf{Open-vocabulary 3D instance segmentation} on ScanNet200, ScanNet++, and MatterPort3D benchmarks, evaluated by Class-agnostic and Semantic segmentation performance.
\textbf{ST} \ding{51}: streaming method. \textbf{ZS} \ding{51}: zero-shot method.
}\label{OV3DIS}
\centering
\renewcommand\tabcolsep{1.0pt}
\resizebox{\linewidth}{!}{
    \begin{tabular}{l|cc|ccc|ccc|ccc|ccc|ccc|ccc}
    \toprule[0.7pt]
    \multirow{3}*{\textbf{Method}} & \multirow{3}*{\textbf{ST}} & \multirow{3}*{\textbf{ZS}} & \multicolumn{6}{c|}{\textbf{ScanNet200}} & \multicolumn{6}{c|}{\textbf{ScanNet++}} & \multicolumn{6}{c}{\textbf{MatterPort3D}}\\ 
    \cline{4-21}
    & & & \multicolumn{3}{c|}{Class-agnostic} & \multicolumn{3}{c|}{Semantic}& \multicolumn{3}{c|}{Class-agnostic} & \multicolumn{3}{c|}{Semantic}& \multicolumn{3}{c|}{Class-agnostic} & \multicolumn{3}{c}{Semantic}\\ 
    \cline{4-21}
          & & & AP & AP$_{50}$ & AP$_{25}$ & AP & AP$_{50}$ & AP$_{25}$ & AP & AP$_{50}$ & AP$_{25}$ & AP & AP$_{50}$ & AP$_{25}$ & AP & AP$_{50}$ & AP$_{25}$ & AP & AP$_{50}$ & AP$_{25}$ \\
          \midrule
          Mask3D~\cite{schult2023mask3d}
          & \ding{55} & \ding{55}
          & 39.7 & 53.6& 62.5 
          & 26.9 & 36.2 & 41.4
          & 22.8 & 33.3 & 45.7 
          & 3.6 & 5.1 & 6.7
          & 4.4 & 9.8 & 20.6
          & 2.5 & 4.5 & 6.7
          \\
          OpenMask3D~\cite{takmaz2023openmask3d} 
          & \ding{55} & \ding{55}
          & 39.7 & 53.6 & 62.5 
          & 15.1& 19.6 & 22.6
          & 22.8 & 33.3 & 45.7
          & 2.0 & 2.7 & 3.4
          & 4.4 & 9.8 & 20.6
          & 4.6 & 8.5 & 13.0
          \\
         OVIR-3D~\cite{Lu2023OVIR3DO3}
         & \ding{55} & \ding{51}
         & 14.4 & 27.5 & 38.8 
         & 9.3 & 18.7& 25.0
         & 19.4 & 34.1 & 46.5
         & 3.6 & 5.7 & 7.3
         & 5.9 & 13.9 &  24.6
         & 6.3 & 16.4 &  24.4
         \\ 
        MaskClustering~\cite{yan2024maskclustering}
        & \ding{55} & \ding{51}
        & 19.7 & 36.4 & 51.4
        & 12.0 & {23.3} & {30.1}
        & 24.6 & {40.3} & 51.5
        & {7.8} & {11.9} & {13.2}
        & 8.3 & 18.9 & 33.4
        & 9.2 & 19.7 & 26.5
        \\ 
        SAI3D~\cite{yin2024sai3d}
        & \ding{55} & \ding{51}
        & 30.8 & 50.5 & 70.6
        & 12.7 & 18.8 & 24.1
        & 17.1 & 31.1 & 49.5
        & --  & -- & --
        & --  & -- & --
        & --  & -- & --
        \\
        SAM-Graph~\cite{guo2024sam}
        & \ding{55} & \ding{55}
        & 22.1 & 41.7 & 62.8
        & --  & -- & --
        & 15.3 & 27.2 & 44.3
        & --  & -- & --
        & --  & -- & --
        & --  & -- & --
        \\
        SAM2Object~\cite{zhao2025sam2object}
        & \ding{55} & \ding{51}
        & 34.0 & 52.7 & 70.3
        & 13.3 & 19.0 & 23.8
        & 20.2 & 34.1 & 48.7
        & --  & -- & --
        & --  & -- & --
        & --  & -- & --
        \\
        VoxelFlow~\cite{wang2025voxelflow}
        & \ding{55} & \ding{51}
        & 21.9 & 39.4 & 55.1
        & 15.0 & 25.5 & 34.2
        & --  & -- & --
        & 6.8 & 11.7 & 15.6
        & --  & -- & --
        & --  & -- & --
        \\
        OE-3DIS~\cite{nguyen2025open}
        & \ding{55} & \ding{51}
        & 16.0 & 22.0 & 24.7
        & --  & -- & --
        & 18.4 & 29.4 & 33.6
        & --  & -- & --
        & --  & -- & --
        & --  & -- & --
        \\
        Octree-Graph~\cite{wang2025open}
        & \ding{55} & \ding{51}
        & --  & -- & --
        & 14.3 & 25.8 & 33.6
        & --  & -- & --
        & --  & -- & --
        & --  & -- & --
        & --  & -- & --
        \\
        MV3DIS~\cite{Zhao_2026_CVPR}
        & \ding{55} & \ding{51}
        & 35.5 & 54.7 & 69.7
        & 15.5 & 20.2 & 24.2
        & 22.0 & 36.7 & 51.7
        & --  & -- & --
        & --  & -- & --
        & --  & -- & --
        \\
        \midrule
        OVIR-3D$_{+\text{Loc2his}}$
        & \ding{51} & \ding{51}
         & 16.2 & 30.6  & 45.4 
         & 8.5  & 16.4  & 22.5
         & 17.4 & 32.5  & 44.8
         & 5.0  & 8.1   & 11.6
         & 6.7  & 16.0  & 32.9
         & 4.8  & 9.5   & 18.6
        \\
        MaskClustering$_{+\text{Loc2his}}$
        & \ding{51} & \ding{51}
        & 17.4 & 31.9 & 51.5
        & 7.5 & 13.8 & 20.8
        & 18.7 & 31.5 & 46.9
        & 4.2 & 7.3 & 10.0
        & 4.9 & 11.6 & 30.1
        & 1.9 & 5.4 & 13.7
        \\
        \textbf{Stream3D (ours)}
        & \ding{51} & \ding{51}
        & {22.4} & {37.6} & {55.1}
        & {13.0} & {21.2} & {28.9}
        & {24.7} & {39.1} & {52.0}
        & 6.6 & {10.0} & {12.5}
        & 8.7 & 19.1 & {39.7}
        & 6.4 & 14.4 & 27.0
        \\
        \textbf{Stream3Dv2 (ours)}
        & \ding{51} & \ding{51}
        & 27.1  & 42.0 & 55.3   
        & 20.6  & 33.7 & 44.6   
        & 27.4  & 41.1 & 53.2
        & 13.5  & 21.3 & 29.1
        & 12.5  & 25.4 & 42.7
        & 12.3  & 25.8 & 40.0
        \\
    \bottomrule[0.7pt]
    \end{tabular}
}
\begin{tablenotes}[flushleft] 
      \footnotesize 
      \item[] ~``${+\text{Loc2his}}$'' indicates the method performs on local segmentation and employs the naive version of our Loc2his module to support streaming segmentation.
\end{tablenotes}
\end{table*}

\begin{table*}[!t]
\caption{\textbf{Per-class segmentation performance comparison} (semantic AP$_{50}$ across the first 20 classes) on ScanNet200 benchmark.
}\label{PCS}
\centering
\renewcommand\tabcolsep{3.0pt}
\resizebox{\linewidth}{!}{
    \begin{tabular}{l|cccccccccccccccccccc|c}
    \toprule[0.7pt]
    & \rotatebox{90}{~chair}    
    & \rotatebox{90}{~table}     
    & \rotatebox{90}{~door} 
    & \rotatebox{90}{~couch}   
    & \rotatebox{90}{cabinet}   
    & \rotatebox{90}{~shelf}  
    & \rotatebox{90}{~desk}
    & \rotatebox{90}{offi.cha.}  
    & \rotatebox{90}{~~bed}       
    & \rotatebox{90}{~pillow}  
    & \rotatebox{90}{~~sink} 
    & \rotatebox{90}{picture}    
    & \rotatebox{90}{window} 
    & \rotatebox{90}{~toilet}
    & \rotatebox{90}{booksh.}  
    & \rotatebox{90}{monitor}    
    & \rotatebox{90}{curtain}  
    & \rotatebox{90}{~book}    
    & \rotatebox{90}{armcha.}    
    & \rotatebox{90}{coff.tab.} 
    & \rotatebox{90}{~~\textbf{Avg.}}
    \\ 
    \midrule
        MaskClustering
        & 36.1  & 35.2 & 0.5  & 20.5  & 2.1  & 8.2 & 4.8   & 19.3 & 26.8 & 21.7  & 37.6 & 2.6 & 6.3   & 32.0 & 2.0 & 44.3  & 23.1 & 2.8 & 16.1  & 31.2 & 18.7
        \\
        Stream3D
        & 38.8  & 36.9 & 0.3
        & 21.9  & 4.3  & 9.1
        & 4.1   & 22.6 & 25.9
        & 15.3  & 29.5 & 1.4
        & 8.4   & 44.6 & 2.0
        & 41.9  & 30.1 & 3.9
        & 19.0  & 44.4 & 20.2
        \\
        Stream3Dv2
        & 40.9  & 49.7 & 50.0
        & 22.5  & 24.2 & 9.1
        & 6.2  & 56.6 & 60.0
        & 23.7  & 41.5 & 35.4
        & 33.7  & 74.6 & 17.5
        & 45.7  & 58.9 & 24.7
        & 46.3  & 45.6 & 38.3
        \\
        \midrule
        \textbf{Improve. $\uparrow$}
        & 2.1  & 12.8 & 49.5
        & 0.6  & 19.9 & 0.0
        & 1.4  & 34.0 & 33.2
        & 2.0  & 3.9 & 32.8
        & 25.3  & 30.0 & 15.5
        & 1.4  & 28.8 & 20.8
        & 27.3  & 1.2 & \textbf{18.1}
        \\
    \bottomrule[0.7pt]
    \end{tabular}
}
\end{table*}

\subsection{Experimental Setup}\label{sec:es}
\textbf{Benchmark datasets.} Following~\cite{yan2024maskclustering,Xu_2026_CVPR}, we conduct the main 3D instance segmentation experiments on widely-used benchmarks including single-room {ScanNet200}~\cite{dai2017scannet,rozenberszki2022language}, high-fidelity {ScanNet++}~\cite{yeshwanth2023scannet++}, large-building {MatterPort3D}~\cite{Matterport3D}.
These datasets consist of real-world indoor scenes, and we follow previous work that uses the validation sets of ScanNet200 (312 scenes, 198 classes, 54090 total frames) and ScanNet++ (50 scenes, 1554 classes, 33884 total frames), as well as the testing set of MatterPort3D (8 scenes, 157 classes, 11052 total frames).
We also follow the setting in \cite{lemeshko2026zoo3d} to perform the 3D object detection on {ScanNet200} dataset.

\textbf{Evaluation metrics.} Following previous work~\cite{yan2024maskclustering,tang2025onlineanyseg,Xu_2026_CVPR}, we evaluate 3D segmentation performance by three metrics AP, AP$_{25}$, and AP$_{50}$.
AP denotes the mean average precision across IoU thresholds from $0.5$ to $0.95$ with $0.05$ intervals.
AP$_{25}$ and AP$_{50}$ are standard average precision under IoU thresholds of $0.25$ and $0.5$, respectively.
We report the class-agnostic and semantic 3D instance segmentation results to comprehensively evaluate the model performance.
By following~\cite{lemeshko2026zoo3d}, AP$_{25}$ and AP$_{50}$ are adopted to evaluate the class-agnostic 3D object detection performance.

\textbf{Comparison baselines.} We compare our method Stream3D and Stream3Dv2 with the representative open-vocabulary 3D segmentation methods including both full-sequence and streaming approaches.
To be specific,
Mask3D~\cite{schult2023mask3d},
OpenMask3D~\cite{takmaz2023openmask3d},
and SAM-Graph~\cite{guo2024sam}
are classical open-vocabulary 3D instance segmentation methods, and they can only operate under the conditions of full-sequence with extra training.
OVIR-3D~\cite{Lu2023OVIR3DO3},
MaskClustering~\cite{yan2024maskclustering},
SAI3D~\cite{yin2024sai3d},
SAM2Object~\cite{zhao2025sam2object},
VoxelFlow~\cite{wang2025voxelflow},
OE-3DIS~\cite{nguyen2025open},
Octree-Graph~\cite{wang2025open},
and MV3DIS~\cite{Zhao_2026_CVPR}
are the state-of-the-art methods in recent years, which have made significant progress in open-vocabulary zero-shot (as known as training-free) 3D instance segmentation.
These methods need the full-sequence RGB-D input and are unable to conduct streaming 3D segmentation.

Currently, the streaming zero-shot 3D instance segmentation (SZ3DS) usually is also training-free and has been rarely studied.
For the full-sequence methods OVIR-3D and MaskClustering, we establish their variants OVIR-3D$_{+\text{Loc2his}}$ and MaskClustering$_{+\text{Loc2his}}$ for SZ3DS, by applying them on local scenes' segmentation and incrementally merging the instance segmentation results with our proposed ${{Loc2his}}$ module.
Additionally, we conduct sufficient literature review and find four streaming 3D segmentation methods for comparison, including
SAM3D~\cite{yang2023sam3d},
EmbodiedSAM~\cite{xuembodiedsam},
OnlineAnySeg~\cite{tang2025onlineanyseg},
and MoonSeg3R~\cite{du2026moonseg3r} and their experiments are limited to class-agnostic 3D instance segmentation on ScanNet200 dataset due to underlying pipeline differences.

For open-vocabulary 3D object detection, we follow the state-of-the-art method~\cite{lemeshko2026zoo3d} and take OpenM3D~\cite{hsu2025openm3d}, Zoo3D$_1$~\cite{lemeshko2026zoo3d}, and Zoo3D$_0$~\cite{lemeshko2026zoo3d} as three baselines.

\textbf{Implementation details.}
We use the same Open3D library adopted in \cite{yan2024maskclustering} to accelerate the processing of point clouds, and employ the same CLIP~\cite{radford2021learning} with ViT-H to implement the image-text matching for open-vocabulary labeling in our Stream3D framework.
For 2D image segmentation models, we employ both SAM2~\cite{ravisam} and SAM3~\cite{carion2026sam} to implement the grid-prompted and semantic-prompted segmentation, respectively.
We adopt the same class words as used in CLIP by previous work~\cite{Lu2023OVIR3DO3,yan2024maskclustering,Xu_2026_CVPR} to form the semantic prompts in the segmentation of SAM3.
The Codex v0.142.0 and gpt-5.5 are adopted to perform open-vocabulary scene understanding and reasoning tasks.
We test our method on 4 NVIDIA GeForce RTX 4090 GPUs on which all samples of a dataset are evenly distributed and processed in parallel.
Evaluation is conducted between the points detected in RGB-D and their ground-truth labels.
For all comparison experiments,
across ScanNet200, ScanNet++, and MatterPort3D, the number of local frames is set to $20$;
the manifold distance is set to $0.05$;
the downsampling rate for selecting key point is $0.05$ in SCP-based noise mask filtering;
and the overlap threshold is $0.2$ in IoU-based key mask merging.
In semantic-driven fine-grained 3D segmentation, $\lambda$ is an implicit and sufficiently large value without requiring setting.
In point cloud manifold refinement, the expected average number of points in voxel is set to $20$, and the maximal iteration $R$ is $5$.

% \footnote{More experiment details can be found in our Appendix and Code which is available at \url{https://github.com/SubmissionsIn/Stream3D} upon acceptance.}.

\begin{figure*}[!ht]
\centering
\includegraphics[width=1\linewidth]{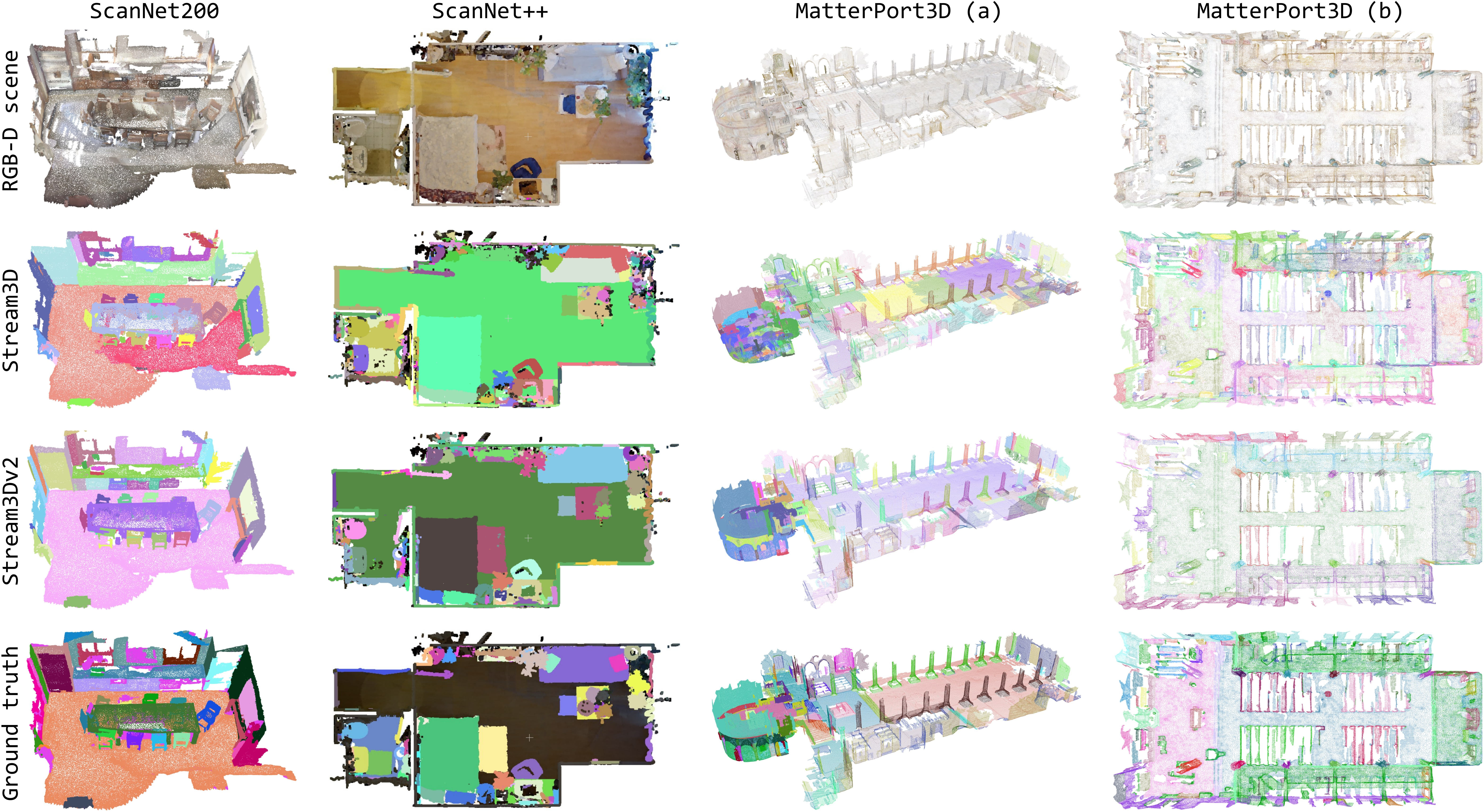}
\caption{Visualization of 3D instance segmentation comparison across RGB-D scene, Stream3D, Stream3Dv2, and Ground truth.}\label{fig:3DS}
\end{figure*}
\begin{figure}[!ht]
\centering
\includegraphics[width=1\linewidth]{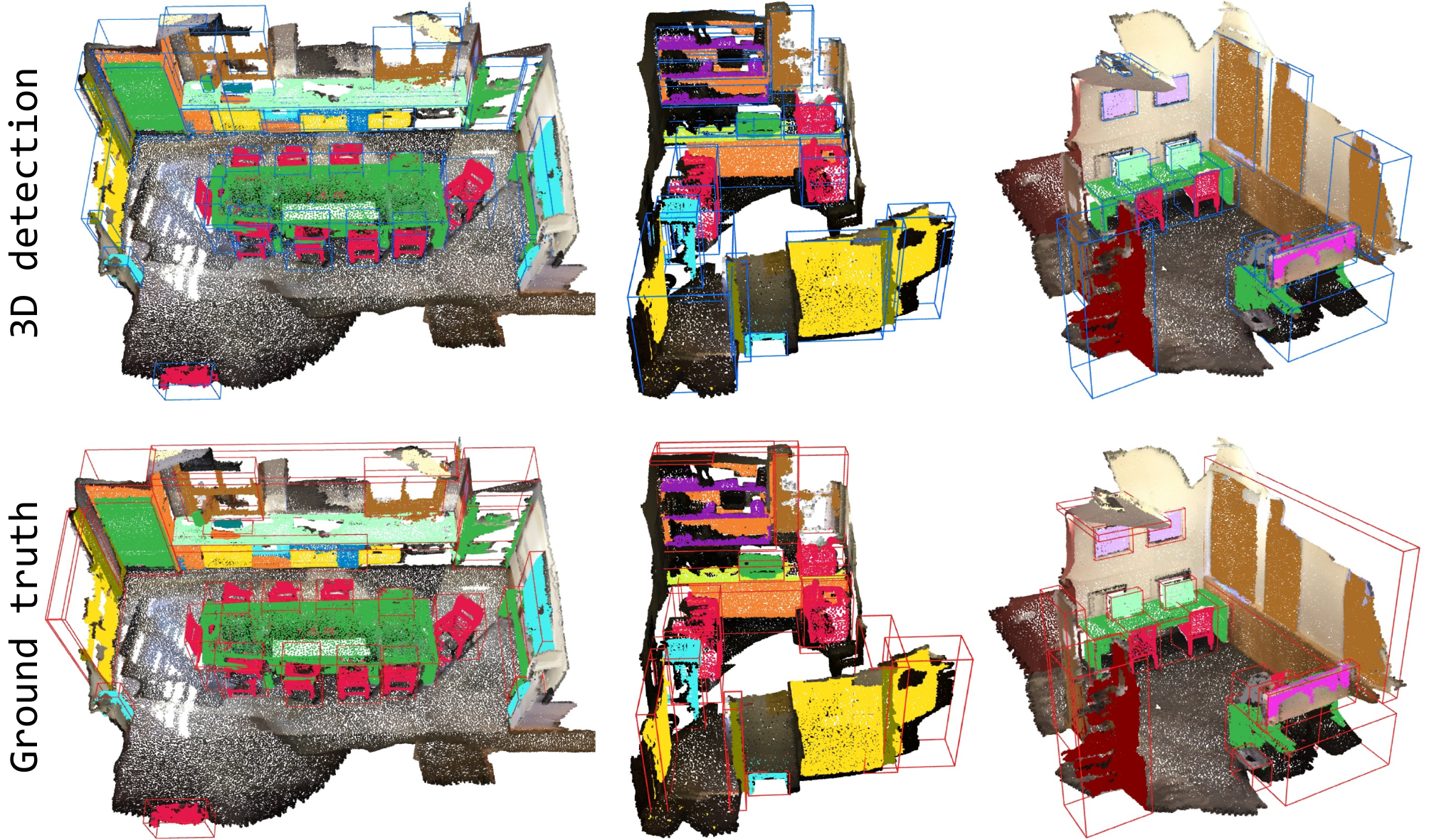}
\caption{Visualization of 3D object detection comparison between our Stream3Dv2 and Ground truth on ScanNet200 benchmark.}\label{fig:3DD}
\end{figure}
\begin{table}[!t]
\caption{\textbf{Class-agnostic streaming 3D instance segmentation} on ScanNet200 benchmark.
\textbf{ST} \ding{51}: streaming method.\\
\textbf{ZS} \ding{51}: zero-shot method.
\textbf{FPS}: frames per second.
}\label{CAS}
\centering
\renewcommand\tabcolsep{6.0pt}
% \resizebox{\linewidth}{!}{
    \begin{tabular}{l|cc|ccc|c}
    \toprule[0.7pt]
    \multirow{2}*{\textbf{Method}} & \multirow{2}*{\textbf{ST}} & \multirow{2}*{\textbf{ZS}} & \multicolumn{4}{c}{\textbf{ScanNet200}} \\ 
    \cline{4-7}
          & & & AP & AP$_{50}$ & AP$_{25}$ & FPS \\
          \midrule
        EmbodiedSAM~\cite{xuembodiedsam}
        & \ding{51} & \ding{55}
        & 28.8  & 42.7 & 54.2
        & 10  
        \\
        SAM3D~\cite{yang2023sam3d}
        & \ding{51} & \ding{51}
        & 9.6 & 24.8 & 49.6
        & 8  
        \\
        OnlineAnySeg~\cite{tang2025onlineanyseg}
        & \ding{51} & \ding{51}
        & 18.6 & 36.1 & 53.5
        & 15 
        \\
        MoonSeg3R~\cite{du2026moonseg3r}
        & \ding{51} & \ding{51}
        & 16.7 & 33.3 & 50.0
        & 15  
        \\
        \textbf{Stream3D (ours)}
        & \ding{51} & \ding{51}
        & {22.4} & {37.6} & {55.1}
        & 13
        \\
        \textbf{Stream3Dv2 (ours)}
        & \ding{51} & \ding{51}
        & 27.1  & 42.0 & 55.3     
        & 11
        \\
    \bottomrule[0.7pt]
    \end{tabular}
% }
\end{table}

\subsection{Main Results}\label{sec:cr}

\textbf{Open-vocabulary 3D instance segmentation.}
Table~\ref{OV3DIS} shows the segmentation performance across three datasets, from which we can find that achieving robust results across different indoor environments is non-trivial due to substantial domain shifts in sensor noise, point density, and scene layouts.
This difficulty is further compounded when jointly imposing streaming perception ($\text{ST}\checkmark$) and zero-shot setting ($\text{ZS}\checkmark$),
where models must dynamically parse unseen object categories frame-by-frame without access to full-sequence point clouds or future temporal context.
As evidenced in Table~\ref{OV3DIS}, zero-shot methods have the advantage of training-free but might suffer drastic performance degeneration without domain-specific knowledge
(\eg for the class-agnostic $\text{AP}$ metric on ScanNet200, OVIR-3D and OE-3DIS drop sharply from the Mask3D's $39.7\%$ to $14.4\%$ and $16.0\%$, respectively).
Furthermore, naive streaming zero-shot baselines still struggle, yielding merely $8.5\%$ ($\text{OVIR-3D}_{+\text{Loc2his}}$) and $7.5\%$ ($\text{MaskClustering}_{+\text{Loc2his}}$) semantic $\text{AP}$ on ScanNet200, uncovering the extreme difficulty of maintaining segmentation fidelity under simultaneous streaming and zero-shot constraints.

\begin{table}[!t]\caption{\textbf{Open-vocabulary 3D object detection} on ScanNet200 benchmark.
\textbf{ST} \ding{51}: streaming method.
\textbf{ZS} \ding{51}: zero-shot method.}\label{OV3DOD}
\small
\centering
\renewcommand\tabcolsep{6.0pt}
% \resizebox{\linewidth}{!}{
    \begin{tabular}{l|cc|cc}
    \toprule[0.7pt]
    \multirow{2}*{\textbf{Method}} & \multirow{2}*{\textbf{ST}} & \multirow{2}*{\textbf{ZS}} & \multicolumn{2}{c}{\textbf{ScanNet200}} \\ 
    \cline{4-5}
        & & & AP$_{50}$ & AP$_{25}$ \\ \midrule
        OpenM3D~\cite{hsu2025openm3d}
        & \ding{55} & \ding{55}
        & -- & 4.2   
        \\
        Zoo3D$_1$~\cite{lemeshko2026zoo3d}
        & \ding{55} & \ding{55}
        & 15.2 & 23.5 
        \\
        Zoo3D$_0$~\cite{lemeshko2026zoo3d}
        & \ding{55} & \ding{51}
        & 14.1 & 21.1 
        \\
        \textbf{Stream3D (ours)}
        & \ding{51} & \ding{51}
        & 10.6 & 18.9
        \\
        \textbf{Stream3Dv2 (ours)}
        & \ding{51} & \ding{51}
        & 14.2 & 24.8  
        \\
    \bottomrule[0.7pt]
    \end{tabular}
% }
\end{table}

Even with these stringent constraints, our framework makes significant progress over comparison methods.
To be specific, in Table~\ref{OV3DIS}, our {Stream3Dv2} outperforms the SOTA zero-shot baseline MV3DIS on ScanNet200 (\eg semantic $\text{AP}$ $20.6\%$ vs. $15.5\%$), despite operating in a streaming setting while MV3DIS requires the full sequence.
On ScanNet++ and MatterPort3D, our {Stream3Dv2} surpasses the best competitors across all evaluation metrics, whether in class-agnostic or semantic segmentation.
Comparison with the specially-designed streaming baselines (also known as online) is shown in Table~\ref{CAS}, where our {Stream3D} and {Stream3Dv2} consistently obtain better performance than streaming zero-shot baselines with the comparable time efficiency.
Moreover, comparing our two versions reveals significant overall gains from {Stream3D} to {Stream3Dv2}, \eg boosting class-agnostic $\text{AP}$ $+4.7\%$, $+2.7\%$, $+3.8\%$ and semantic $\text{AP}$ $+7.6\%$, $+6.9\%$, $+5.9\%$ on the three benchmarks, indicating our new method Stream3Dv2 has achieved substantial improvement.

To further inspect this improvement, Table~\ref{PCS} details the per-class semantic $\text{AP}_{50}$ performance across the first 20 categories on ScanNet200.
{Stream3Dv2} elevates the average semantic $\text{AP}_{50}$ from $20.2\%$ ({Stream3D}) and $18.7\%$ ({MaskClustering}) to $38.3\%$, yielding a massive average improvement of $+18.1\%$.
Notable performance leaps are observed across challenging geometry classes and fine-grained structures, such as \textit{office chair} ($22.6\%\to56.6\%$), \textit{bed} ($26.8\%\to60.0\%$), \textit{picture} ($2.6\%\to35.4\%$), and \textit{toilet} ($44.6\%\to74.6\%$).
This substantial per-class margin validates that the architectural enhancements in {Stream3Dv2} effectively boost streaming semantic association.
Fig.~\ref{fig:3DS} visualizes the segmentation comparison between Stream3D and Stream3Dv2 on a single-room scene in ScanNet200, a high-fidelity scene in ScanNet++, and two large-building scenes in MatterPort3D, suggesting that our Stream3Dv2 owns better temporal geometry consistency.

\begin{figure}[!t]
\centering
\includegraphics[width=1\linewidth]{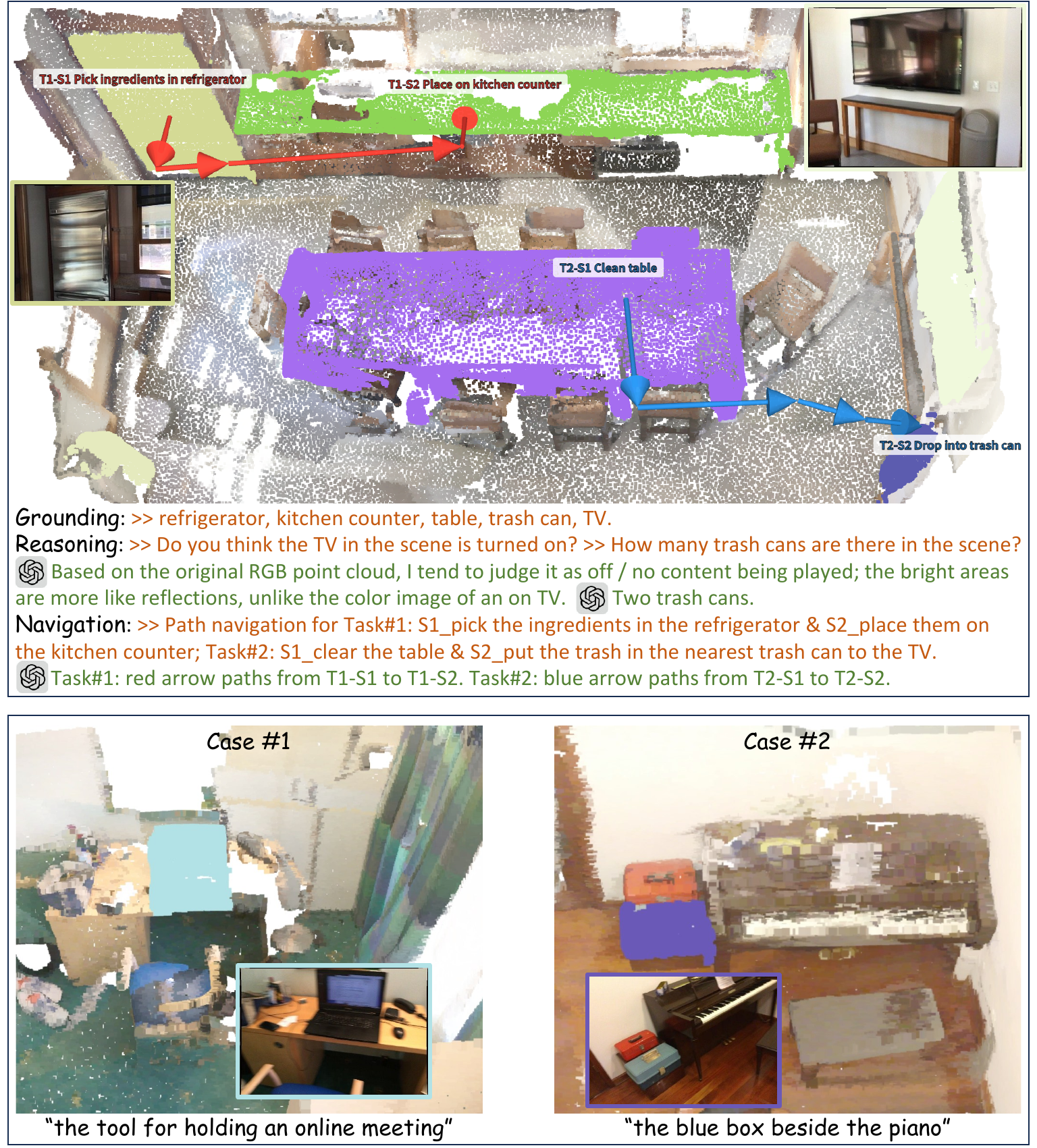}
\caption{Examples of downstream tasks by Stream3Dv2 equipped with LLM agent. \emph{Top}: 3D instance grounding, reasoning, and navigation. \emph{Bottom}: open-vocabulary queries with indirect semantic descriptions.
}\label{fig:agent}
\end{figure}

\textbf{Open-vocabulary 3D object detection.}
Beyond instance segmentation, Table~\ref{OV3DOD} evaluates open-vocabulary 3D object detection performance on ScanNet200.
Fig.~\ref{fig:3DD} directly shows the detection results by our Stream3Dv2 which even obtains more reasonable 3D bounding boxes than annotations.

Under the dual constraints of zero-shot setting ($\text{ZS}\checkmark$) and streaming perception ($\text{ST}\checkmark$), our {Stream3Dv2} achieves $14.2\%$ $\text{AP}_{50}$ and $24.8\%$ $\text{AP}_{25}$.
This demonstrates that our framework effectively consolidates geometric and semantic information into precise 3D bounding boxes in streaming manner.
Compared to the zero-shot full-sequence baseline $\text{Zoo3D}_0$, our {Stream3Dv2} creates a significant improvement of $+3.7\%$ in $\text{AP}_{25}$ ($21.1\%\to24.8\%$).
Remarkably, {Stream3Dv2} even surpasses the non-zero-shot full-sequence model $\text{Zoo3D}_1$ in terms of $\text{AP}_{25}$ by $+1.3\%$ ($23.5\%\to24.8\%$).
Consistent with our segmentation results, {Stream3Dv2} brings substantial gains over {Stream3D}, boosting $\text{AP}_{50}$ from $10.6\%$ to $14.2\%$ ($+3.6\%$) and $\text{AP}_{25}$ from $18.9\%$ to $24.8\%$ ($+5.9\%$).
These consistent advancements across both segmentation and detection tasks highlight the robust architectural design of Stream3Dv2.

\textbf{3D scene understanding with LLM-based agent.}
We integrate Stream3Dv2 with Codex (v0.142.0, gpt-5.5) as an LLM agent to further validate the 3D scene understanding ability.
The agent uses RGB values and 3D segmentation point cloud files obtained by Stream3Dv2 as the input, combined with the prompt instructions.
This can support the 3D instance grounding, scene reasoning, path navigation, and fuzzy query tasks as demonstrated in Fig.~\ref{fig:agent}.
The top panel showcases that it accurately reasons about fine-grained object states (\eg TV display status, number of trash cans) and grounds sequential navigation paths (\eg T1-S1 to T1-S2).
In the bottom panel, the cases use indirect semantic queries for open-vocabulary 3D instance segmentation, highlighting the ability of Stream3Dv2 to support language-driven scene understanding and reasoning.
By bridging 3D spatial perception with LLM understanding, our Stream3Dv2 overcomes the constraints of closed-vocabulary models, demonstrating application potential in downstream tasks for embodied agents.

\begin{figure}[!t]
	\centering 
	\includegraphics[width=0.49\textwidth]{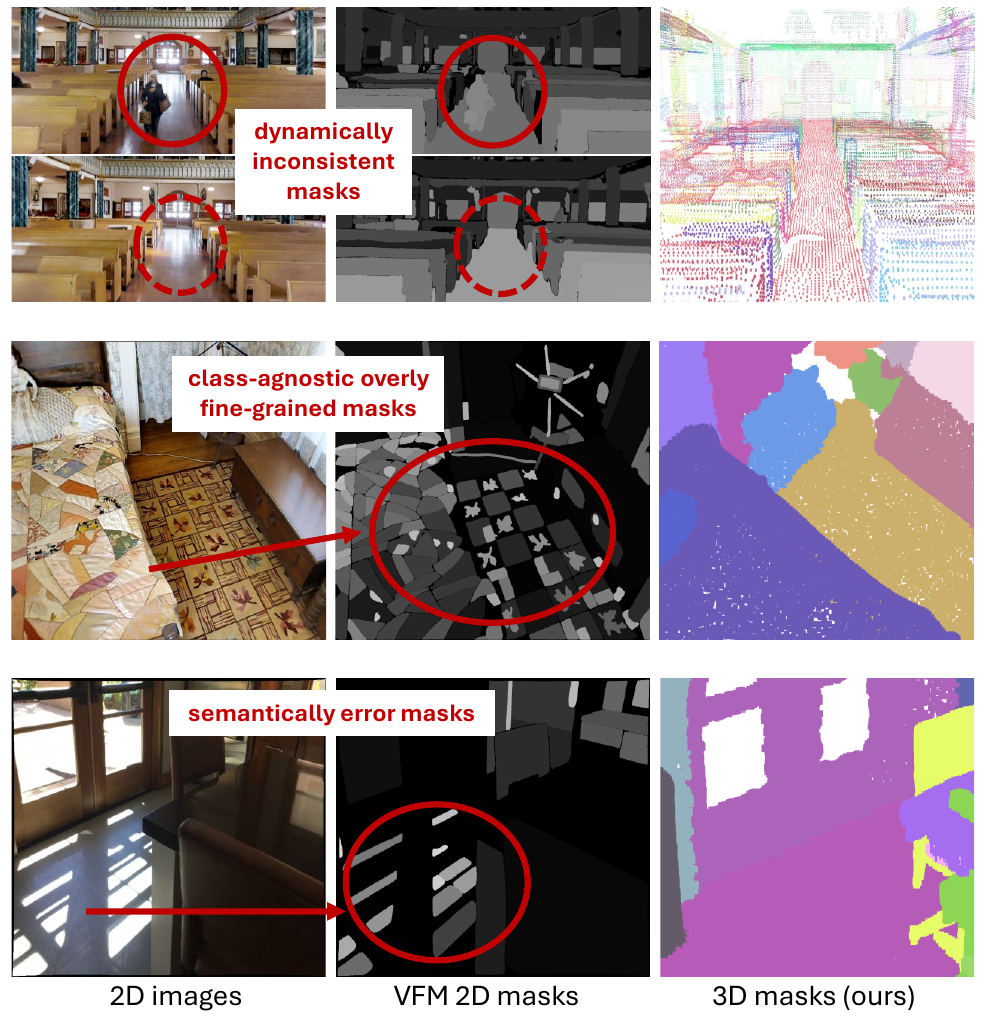}
    \caption{Our method obtains robust 3D segmentation by mitigating 2D noise masks from VFMs, in cases such as dynamically inconsistent masks, class-agnostic fine-grained masks, semantically error masks.}\label{fig:noise_robustness}
\end{figure}

\begin{figure*}[t]
\centering
\includegraphics[width=1\linewidth]{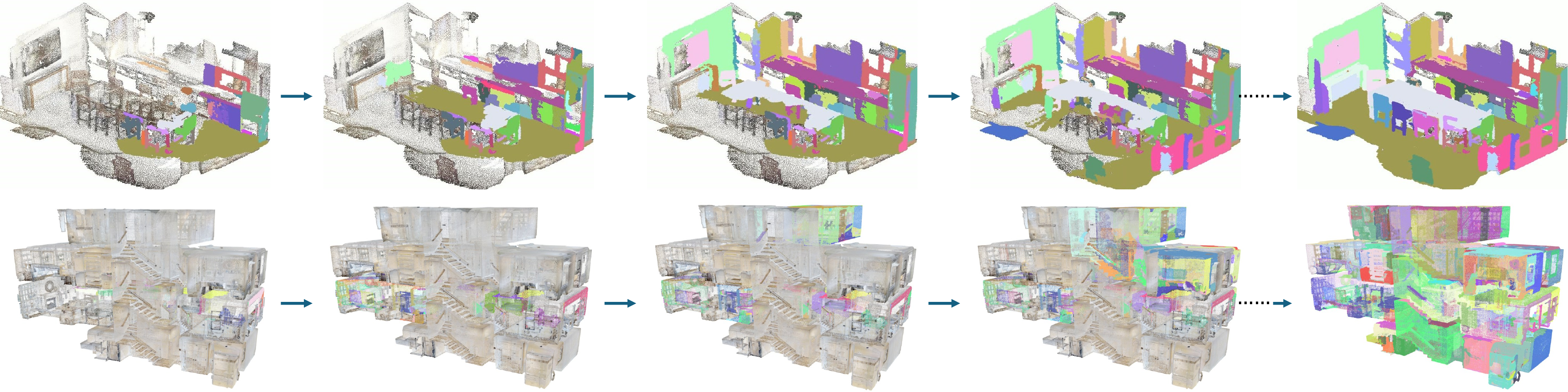}
\caption{The process of our streaming zero-shot 3D instance segmentation on \emph{Top}: room-level cases and \emph{Bottom}: large building-level cases.}\label{fig:Case}
\end{figure*}

\begin{table*}[!t]
\caption{
\textbf{Ablation study on the main methodologies in Stream3Dv2 paradigm.}
\emph{\textbf{MVF}} denotes the multi-view noise mask filtering. \\
\emph{\textbf{SDS}} is the semantic-driven fine-grained 3D segmentation.
\emph{\textbf{PMR}} is the point cloud manifold refinement.
}\label{abl}
\centering
\renewcommand\tabcolsep{2.0pt}
\resizebox{\linewidth}{!}{
    \begin{tabular}{l|ccc|ccc|ccc|ccc|ccc|ccc|ccc}
    \toprule[0.7pt]
    &\multirow{3}*{\emph{\textbf{MVF}}} & \multirow{3}*{\emph{\textbf{SDS}}} & \multirow{3}*{\emph{\textbf{PMR}}} & \multicolumn{6}{c|}{\textbf{ScanNet200}} & \multicolumn{6}{c|}{\textbf{ScanNet++}} & \multicolumn{6}{c}{\textbf{MatterPort3D}}\\ 
    \cline{5-22}
    &&&& \multicolumn{3}{c|}{Class-agnostic} & \multicolumn{3}{c|}{Semantic}& \multicolumn{3}{c|}{Class-agnostic} & \multicolumn{3}{c|}{Semantic}& \multicolumn{3}{c|}{Class-agnostic} & \multicolumn{3}{c}{Semantic}\\ 
    \cline{5-22}
        &&& & AP & AP$_{50}$ & AP$_{25}$ & AP & AP$_{50}$ & AP$_{25}$ & AP & AP$_{50}$ & AP$_{25}$ & AP & AP$_{50}$ & AP$_{25}$ & AP & AP$_{50}$ & AP$_{25}$ & AP & AP$_{50}$ & AP$_{25}$ \\ 
        \midrule
        (I) &\ding{55} & \ding{55} & \ding{55} 
        & 16.6 & 28.7 & 45.7
        & 8.0  & 14.3 & 21.5
        & 5.0  & 9.6  & 19.4
        & 1.7  & 3.5  & 5.3
        & 3.0  & 7.0  & 18.4
        & 1.6  & 3.6  & 10.7
        \\
        (II) &\ding{51}&  & 
        & 19.0 & 32.4 & 49.7
        & 9.6  & 16.1 & 22.6
        & 11.3 & 22.1 & 39.8
        & 3.4  & 6.2  & 9.8
        & 5.6  & 13.1 & 33.3
        & 4.3  & 11.7 & 23.2
        \\
        (III) && \ding{51} & 
        & 20.1 & 33.9 & 48.9
        & 17.1 & 28.9 & 42.0
        & 18.3 & 30.8 & 43.3
        & 12.8 & 20.6 & 28.4
        & 10.6 & 21.3 & 39.9
        & 10.4 & 22.5 & 38.7
        \\
        (IV) & & & \ding{51}
        & 18.6 & 32.2 & 48.3
        & 8.8  & 16.0  & 28.5
        & 14.1 & 23.4 & 35.1
        & 4.5  & 6.4  & 8.0
        & 8.0 & 17.1 & 33.9
        & 2.5 & 4.8  & 12.6
        \\
        (V) & \ding{51} & \ding{51} & 
        & 24.7 & 39.3 & 53.5
        & 17.8 & 29.5 & 43.6
        & 21.9 & 36.3 & 51.3
        & 13.6 & 21.6 & 29.2
        & 11.4 & 23.4 & 41.8
        & 11.7 & 23.3 & 39.6
        \\
        (VI) & \ding{51} &  &  \ding{51}
        & 22.1 & 37.4 & 55.1
        & 12.6 & 20.8 & 28.1
        & 24.6 & 39.0 & 52.4
        & 6.4  & 9.4  & 12.1
        & 8.8  & 19.1 & 40.1
        & 7.7  & 16.8 & 27.7
        \\
        (VII) & & \ding{51} &  \ding{51}
        & 26.7 & 41.1 & 54.4
        & 18.2 & 30.5 & 43.9
        & 26.2 & 39.6 & 51.8
        & 13.2 & 20.6 & 29.2
        & 12.6 & 25.0 & 42.3
        & 11.5 & 25.1 & 38.5
        \\
        (VIII) & \ding{51} & \ding{51} & \ding{51} 
        & 27.1  & 42.0 & 55.3   
        & 20.6  & 33.7 & 44.6   
        & 27.4  & 41.1 & 53.2
        & 13.5  & 21.3 & 29.1
        & 12.5  & 25.4 & 42.7
        & 12.3  & 25.8 & 40.0
        \\
    \bottomrule[0.7pt]
    \end{tabular}
}
\end{table*}

\begin{table}[!ht]
\caption{\textbf{Time latency} during long sequence segmentation on MatterPort3D benchmark (scene id: gxdoqLR6rwA).}\label{ST}
\centering
\renewcommand\tabcolsep{4pt} % 调整表格列间的长度
\begin{threeparttable}
    \begin{tabular}{l|c|c|c|c}
    \toprule[0.7pt]
     & $t=400$ & $t=1200$ & $t=2000$ & $t=2800$ \\
    \midrule
    Loc 3D seg. (ms) & 52  &  67  &  43 &  73 \\
    Loc2his 3D seg. (ms) & 91  &  241 & 237 & 377 \\
    Memory (MB) & 10.4  & 15.6   & 19.0  & 21.3  \\
    \bottomrule[0.7pt]
    \end{tabular}
\end{threeparttable}
\end{table}

\subsection{Model Analysis}\label{sec:ma}
\textbf{Robustness from 2D to 3D.}
To demonstrate the resilience of our framework, Fig.~\ref{fig:noise_robustness} provides qualitative comparisons across challenging scenarios where 2D Vision Foundation Models (VFMs) suffer from severe mask degradation.
Specifically, 2D VFMs frequently produce noise segmentation masks under real-world conditions, including but not limited to
\emph{Top:} dynamically inconsistent masks across frames caused by transient objects or moving pedestrians.
\emph{Middle:} class-agnostic, overly fine-grained masks that over-segment continuous surfaces into fragmented patches (\eg patterned carpets).
\emph{Bottom:} semantically erroneous masks triggered by complex lighting and shadow reflections.
Despite these corrupted 2D inputs, our 3D segmentation approach (right column in Fig.~\ref{fig:noise_robustness}) effectively filters out frame-wise noise masks and resolves multi-view semantic ambiguities.
By enforcing spatial continuity and leveraging semantic consistency, our method produces clean, holistic, and semantically accurate 3D instance masks, demonstrating strong robustness against imperfect 2D priors.

\textbf{Streaming 3D segmentation.}
Fig.~\ref{fig:Case} illustrates the streaming 3D segmentation process of our method, offering insight into its underlying mechanisms. Starting from an initial local perspective, our approach incrementally segments newly discovered instances. For partially observed objects, the instance masks are dynamically refined as the scene stream evolves, closely mirroring how humans progressively perceive unfamiliar environments. Notably, owing to its light computational overhead, our method scales seamlessly from single-room settings to large-scale building environments.
To further demonstrate this scalable efficiency, Table~\ref{ST} reports the time latency on a long-sequence scene. Our local segmentation part maintains the stable execution time ($<100\text{ ms}$), while the memory footprint of the mask pool and the local-to-historical segmentation part grow linearly as the number of RGB-D frames increases ($t$ denotes the $t$-th frame), highlighting its suitability for long-sequence streaming perception tasks.

\begin{figure*}[!t]
\centering
  \begin{subfigure}{0.24\linewidth}
    \includegraphics[width=\linewidth]{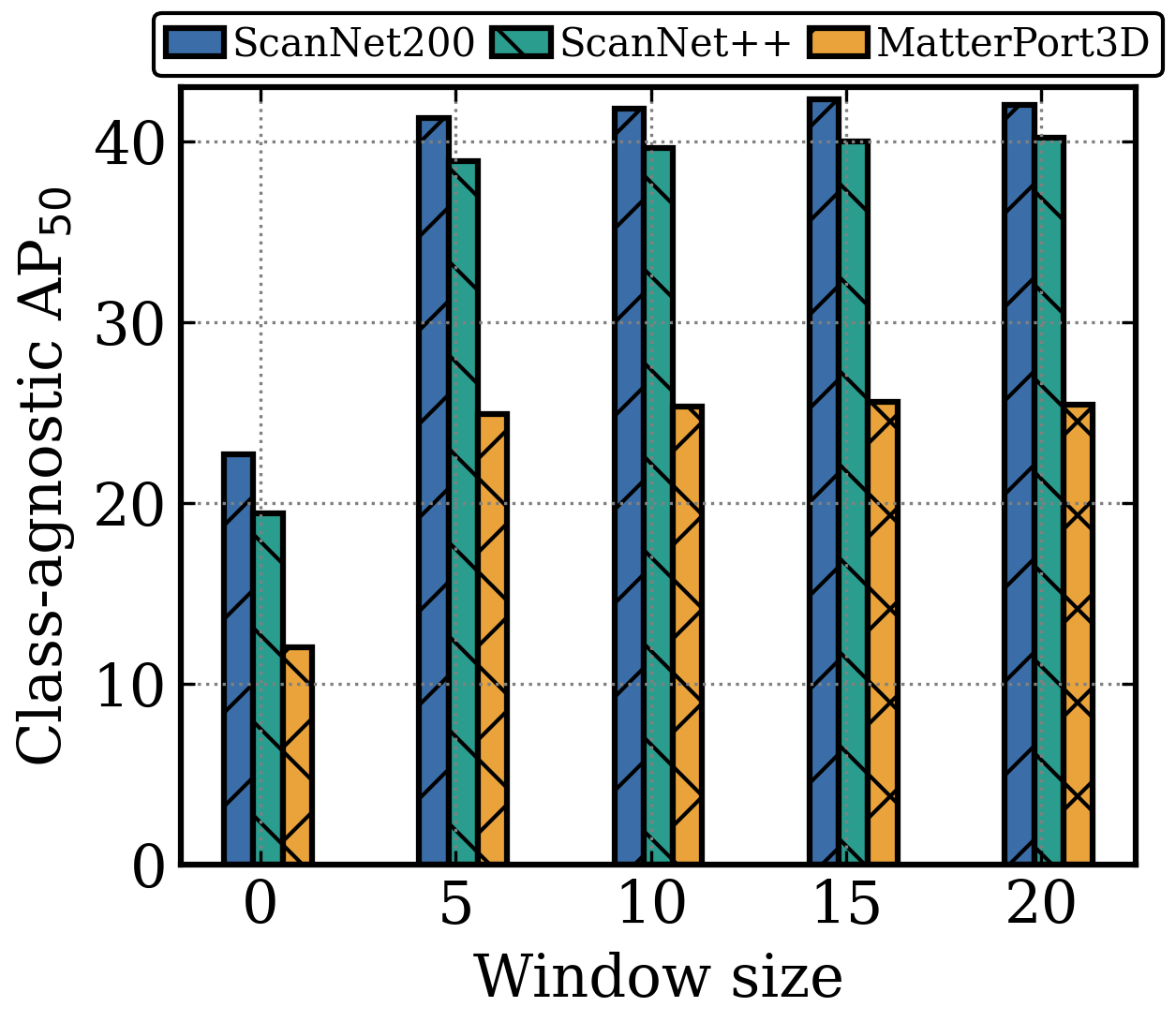}
    \caption{}
  \end{subfigure}
  \begin{subfigure}{0.249\linewidth}
    \includegraphics[width=\linewidth]{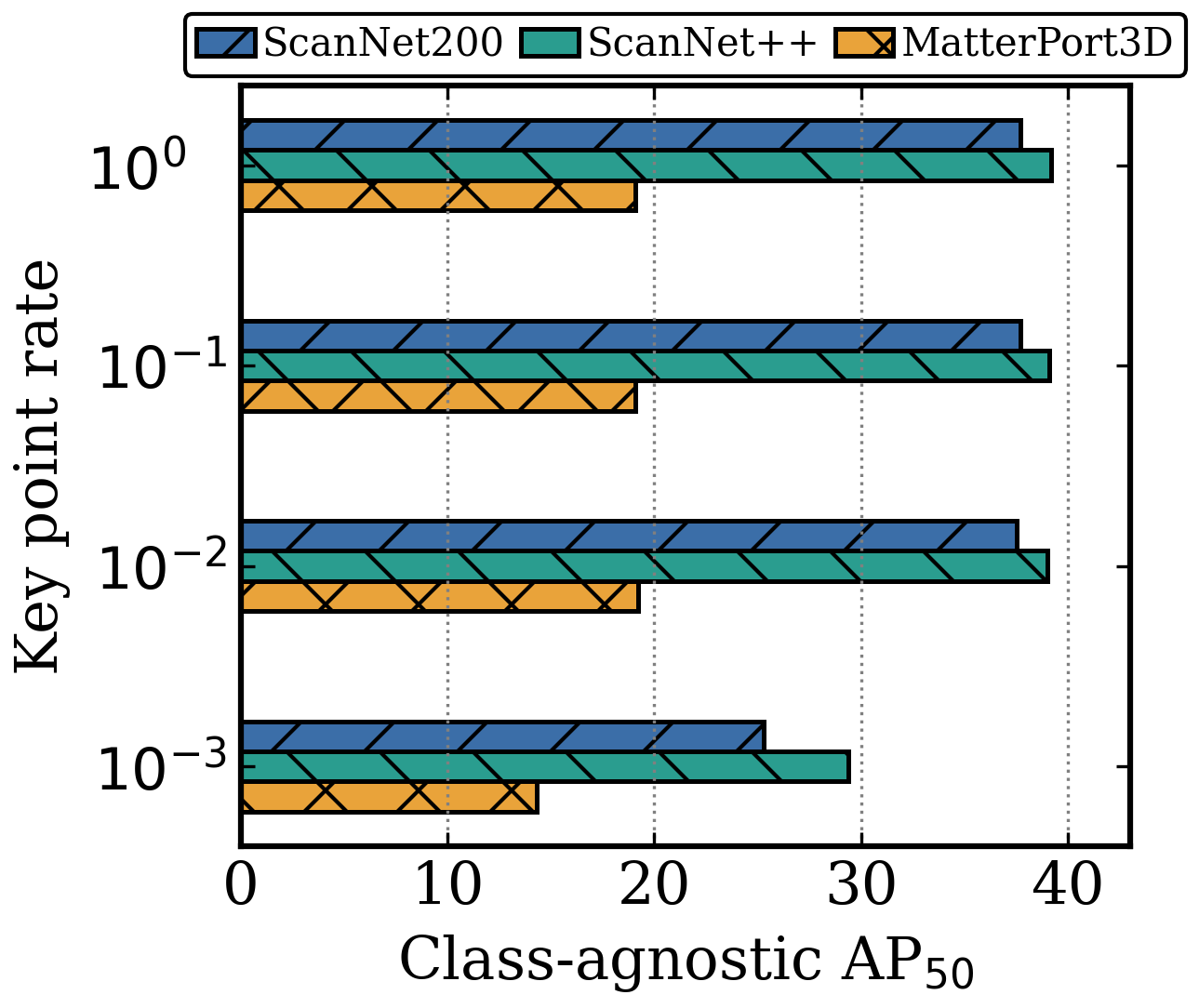}
    \caption{}
  \end{subfigure}
  \begin{subfigure}{0.235\linewidth}
    \includegraphics[width=\linewidth]{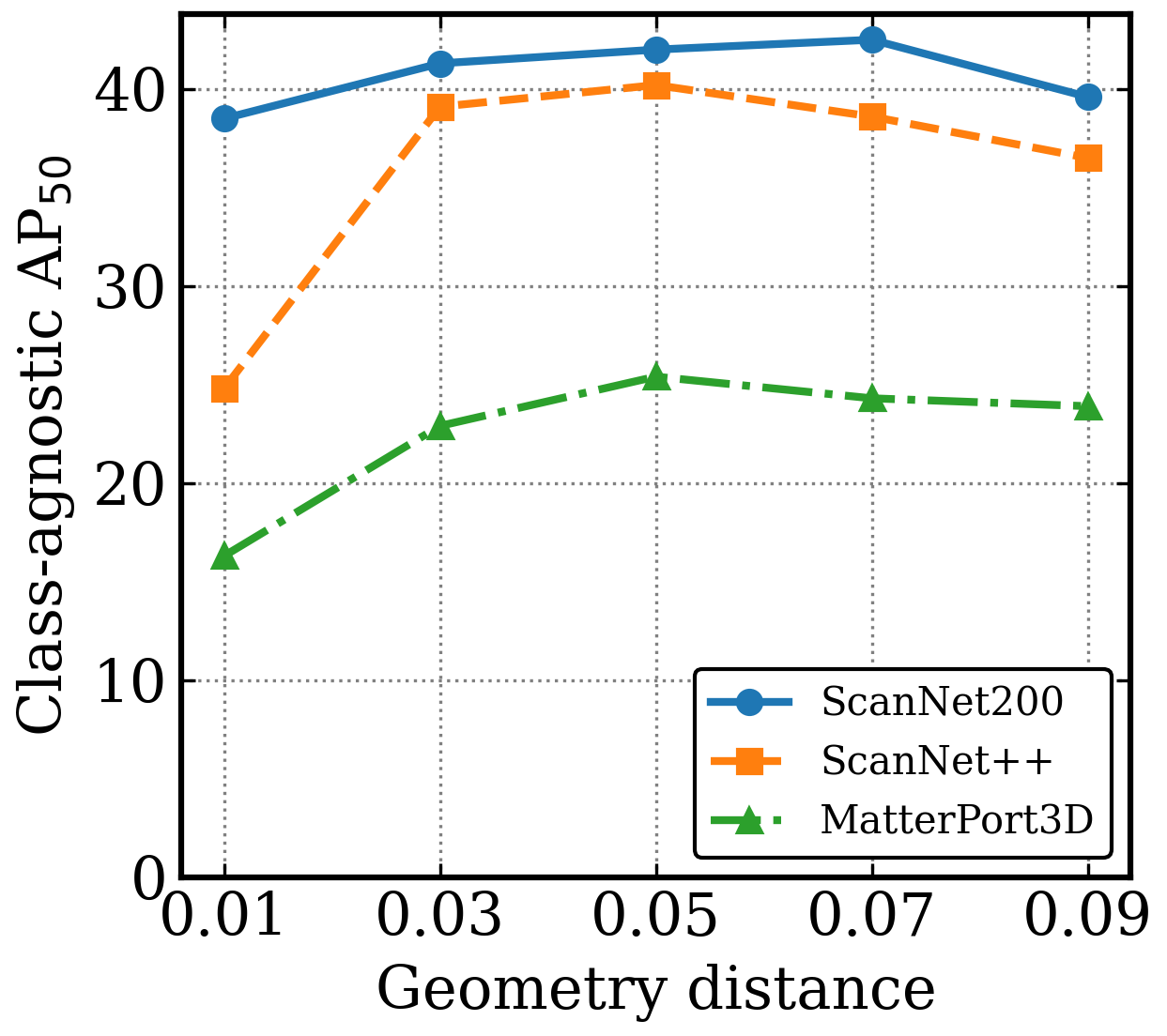}
    \caption{}
  \end{subfigure}
  \begin{subfigure}{0.235\linewidth}
    \includegraphics[width=\linewidth]{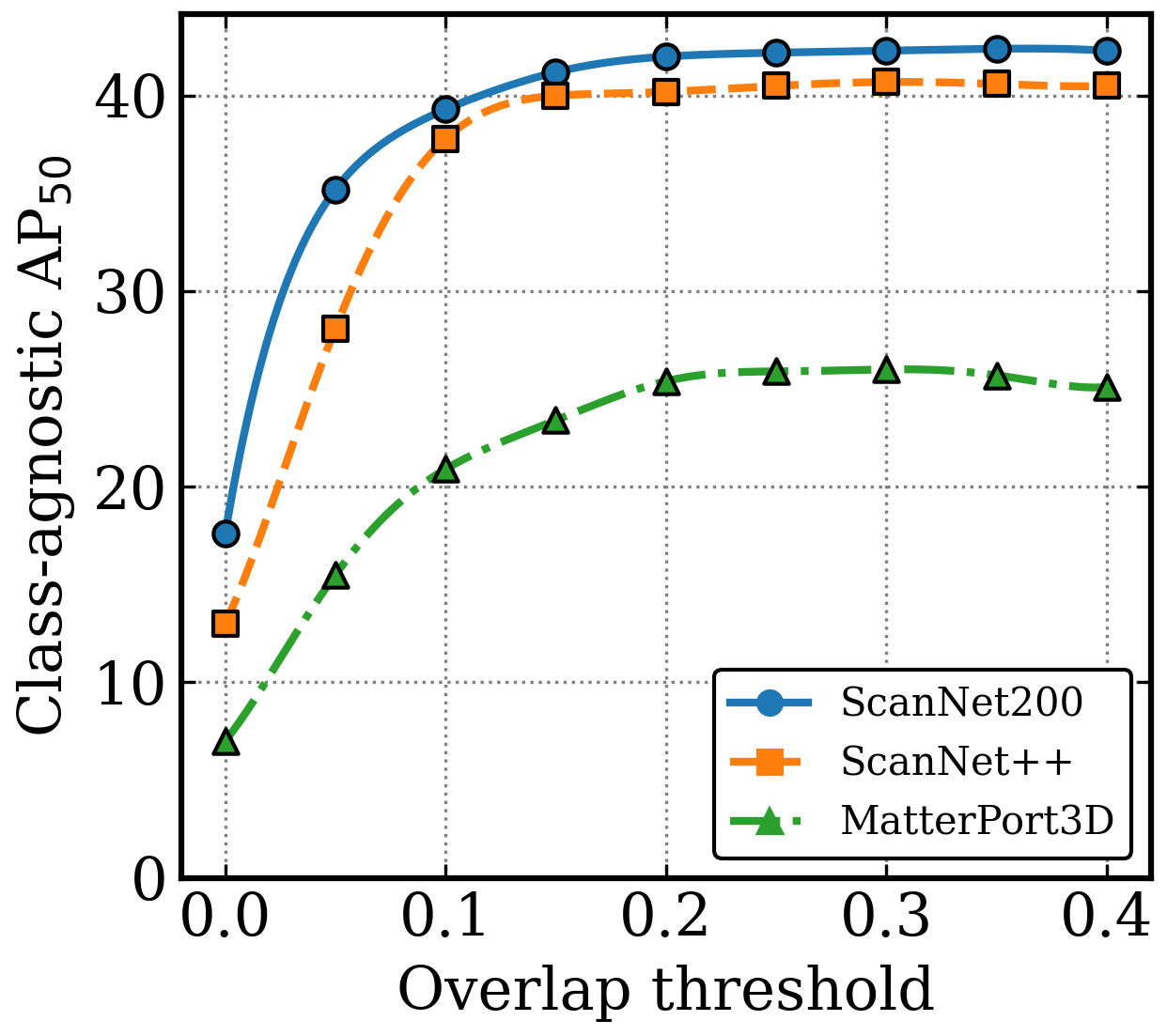}
    \caption{}
  \end{subfigure}
\caption{Parameter analysis on (a) window size $k$, (b) key point rate $\gamma$, (c) geometry distance $\delta$, and (d) overlap threshold $\alpha$.}\label{para}
\end{figure*}

\begin{table*}[!t]
\caption{
\textbf{Ablation Study on two kinds of 2D image segmentation prompts} in Local Multi-View Coarse-Grained 3D Segmentation.
}\label{prompt}
\centering
\renewcommand\tabcolsep{3.0pt}
\resizebox{\linewidth}{!}{
    \begin{tabular}{l|ccc|ccc|ccc|ccc|ccc|ccc}
    \toprule[0.7pt]
    \multirow{3}*{\textbf{Prompt}}  & \multicolumn{6}{c|}{\textbf{ScanNet200}} & \multicolumn{6}{c|}{\textbf{ScanNet++}} & \multicolumn{6}{c}{\textbf{MatterPort3D}}\\ 
    \cline{2-19}
    & \multicolumn{3}{c|}{Class-agnostic} & \multicolumn{3}{c|}{Semantic}& \multicolumn{3}{c|}{Class-agnostic} & \multicolumn{3}{c|}{Semantic}& \multicolumn{3}{c|}{Class-agnostic} & \multicolumn{3}{c}{Semantic}\\ 
    \cline{2-19}
        & AP & AP$_{50}$ & AP$_{25}$ & AP & AP$_{50}$ & AP$_{25}$ & AP & AP$_{50}$ & AP$_{25}$ & AP & AP$_{50}$ & AP$_{25}$ & AP & AP$_{50}$ & AP$_{25}$ & AP & AP$_{50}$ & AP$_{25}$ \\ 
        \midrule
        semantic
        & 14.0 & 26.5 & 44.3
        & 16.0 & 29.0 & 42.6
        & 13.7 & 26.1 & 43.2
        & 9.1  & 15.8 & 25.5
        & 9.1  & 18.8 & 32.4
        & 11.5 & 24.4 & 38.6
        \\
        grid
        & 22.4 & 37.6 & 55.1
        & 13.0 & 21.2 & 28.9
        & 24.7 & 39.1 & 52.0
        & 6.6  & 10.0 & 12.5
        & 8.7  & 19.1 & 39.7
        & 6.4  & 14.4 & 27.0
        \\
        semantic + grid
        & 27.1  & 42.0 & 55.3   
        & 20.6  & 33.7 & 44.6   
        & 27.4  & 41.1 & 53.2
        & 13.5  & 21.3 & 29.1
        & 12.5  & 25.4 & 42.7
        & 12.3  & 25.8 & 40.0
        \\
    \bottomrule[0.7pt]
    \end{tabular}
}
\end{table*}

\textbf{Ablation study.} To validate the core designs in the proposed Stream3Dv2, in Table~\ref{abl}, we evaluate the individual and joint contributions of \emph{Multi-View Noise Mask Filtering (MVF)}, \emph{Semantic-Driven Fine-Grained 3D Segmentation (SDS)}, and \emph{Point Cloud Manifold Refinement (PMR)} across ScanNet200, ScanNet++, and MatterPort3D.
Comparing the baseline (I) with single-module variants (II-IV) reveals clear orthogonal benefits.
\emph{MVF} explores geometric consistency across multiple views to suppress early noise masks, boosting class-agnostic segmentation (\eg class-agnostic $\text{AP}_{50}$ from $28.7\%$ to $32.4\%$ on ScanNet200).
\emph{SDS} serves as the primary driver during the geometric-semantic fusion, markedly increasing semantic segmentation (\eg semantic $\text{AP}$ from $1.7\%$ to $12.8\%$ on ScanNet++).
\emph{PMR} refines point cloud boundaries based on non-Euclidean metrics, elevating 3D mask accuracy (\eg class-agnostic $\text{AP}_{25}$ from $18.4\%$ to $33.9\%$ on MatterPort3D).
Pairwise combinations (V-VII) show consistent performance gains over individual modules.
Integrating all components (VIII) achieves the highest scores across almost all metrics (\eg $27.1\%$ class-agnostic $\text{AP}$ and $20.6\%$ semantic $\text{AP}$ on ScanNet200), confirming the complementary nature of MVF, SDS, and PMR in bridging 2D priors with 3D instances.

\begin{table}[!t]
\caption{\textbf{Ablation study on mask localization strategies (MLS)} in local-to-historical streaming segmentation. The reported results are average time cost for each RGB-D frame (ms).}\label{mls}
\centering
\renewcommand\tabcolsep{4.0pt}
\resizebox{\linewidth}{!}{
    \begin{tabular}{l|c|c|c}
    \toprule[0.7pt]
     \multicolumn{1}{c|}{\textbf{MLS}} & \multicolumn{1}{c|}{\textbf{ScanNet200}} & \multicolumn{1}{c|}{\textbf{ScanNet++}} & \multicolumn{1}{c}{\textbf{MatterPort3D}}\\ 
        \midrule
        Normal  
        & 0.041  & 0.234 & 0.826
        \\
        Detection-based 
        & 0.038 ($\downarrow 7.3\%$)  & 0.196 ($\downarrow 16.2\%$) & 0.423 ($\downarrow 48.8\%$)
        \\
    \bottomrule[0.7pt]
    \end{tabular}
}
\end{table}

Table~\ref{prompt} investigates different prompting mechanisms in our local multi-view segmentation module.
Pure semantic prompts provide class-level guidance but suffer from lower spatial coverage on 2D images, yielding limited 3D segmentation for the whole scene (\eg only $14.0\%$ class-agnostic $\text{AP}$ on ScanNet200).
Conversely, dense grid prompts significantly enhance object identification on 2D images (\eg boosting class-agnostic $\text{AP}$ to $22.4\%$ on ScanNet200), yet sacrifice semantic discrimination (\eg semantic $\text{AP}$ drops from $16.0\%$ to $13.0\%$).
The hybrid semantic$+$grid design achieves the optimal balance, outperforming single-prompt baselines across all metrics (\eg achieving $27.1\%$ class-agnostic $\text{AP}$ and $20.6\%$ semantic $\text{AP}$ on ScanNet200).
This demonstrates that marrying dense geometry segmentation with explicit semantic priors on 2D images is essential for achieving robust 3D segmentation.

Table~\ref{mls} evaluates the average per-frame time cost of mask localization strategies (MLS) in our local-to-historical streaming segmentation module.
Compared to the baseline Normal MLS, our proposed Detection-based MLS substantially reduces processing latency, achieving speedups of $7.3\%$ on ScanNet200, $16.2\%$ on ScanNet++, and $48.8\%$ on MatterPort3D.
By leveraging compact 3D bounding box representations to simplify spatial association, our Detection-based MLS effectively eliminates redundant point-level computations during streaming segmentation.
These results demonstrate that our method achieves superior computational efficiency, making it highly suitable for near-real-time 3D streaming perception.

\textbf{Parameter analysis.}
In Fig.~\ref{para}, we analyze the four parameters in our streaming framework, \ie the number of frames $k$ in noise mask filtering, the rate $\gamma$ for downsampling key points, the geometry distance $\delta$ in point cloud manifold constraints, and the overlap threshold $\alpha$ in IoU-based key mask merging.
In Fig.~\ref{para}(a), we observe that the number of local frames $k$ is insensitive to our framework, and this could be the reason why our method performs well in both dense (ScanNet200 and ScanNet++) and sparse (MatterPort3D) frame conditions. We set $k=20$ by default across all benchmarks.
In Fig.~\ref{para}(b), the performance stays nearly constant across a wide range of sampling rates $\gamma \in [10^{-2}, 10^0]$, but experiences a sharp decline when $\gamma$ drops to $10^{-3}$ because overly sparse key points provide insufficient coverage. Our experiments empirically use $\gamma = 0.05$ to select key points.
In Fig.~\ref{para}(c), the manifold distance $\delta$ produces stable results between $0.03$ and $0.07$ on three datasets.
We fixed $\delta=0.05$ as the default setting and this parameter could vary with the density of the point cloud in practical applications.
In Fig.~\ref{para}(d), the overlap threshold $\alpha$ generally achieves stable results between $0.1$ and $0.3$ on different datasets.
By default, we fixed $\alpha=0.2$ in experiments.

\section{Conclusion}
This paper presents Stream3Dv2, a comprehensive training-free methodology for open-vocabulary streaming zero-shot 3D scene understanding, which substantially extends and elevates our preliminary work Stream3D.
To tackle the fundamental challenges of multi-view 2D noise masks, visual ambiguity, and point cloud boundary misalignment, Stream3Dv2 introduces a suite of newly engineered components: semantic-\&grid-prompted RGB image segmentation, semantic-driven mask merging\&partitioning, and manifold-aware graph metric-based point cloud manifold refinement.
Compared to previous Stream3D, these extensions yield consistent performance improvements, particularly in multi-view semantic consistency, geometric manifold refining, and open-vocabulary grounding.
Extensive evaluations across ScanNet200, ScanNet++, and MatterPort3D demonstrate that Stream3Dv2 not only far surpasses existing streaming baselines, but also achieves performance highly competitive with SOTA offline methods.

Furthermore, the newly added modules expand the applicability of Stream3Dv2 well beyond 3D instance segmentation, empowering diverse downstream applications such as 3D object detection and LLM-driven embodied agent interaction.
Despite these significant advancements, scaling streaming 3D perception to unconstrained, ultra-large outdoor environments with dynamic occlusions remains an open challenge.
Future research will explore adapting our Stream3Dv2 framework to outdoor scene understanding and further coupling it with physical embodied agents for real-time robotic manipulation.

{
\small
\bibliographystyle{IEEEtran}
\bibliography{main}
}

\onecolumn
\input{sec/Appendix}

\end{document}

%% file: sec/Appendix.tex
\section*{Appendix}
Due to the paper space limitations, this appendix provides more contents including algorithm implementation, complexity analysis, convergence analysis, and more experimental results.
\subsection{Algorithm Implementation}
This part introduces some specific implementation details for our algorithm Stream3Dv2, including Geometry-based single-view mask denoising, SCP-based multi-view noise mask filtering, and Closed-form solution for set partitioning.

\textbf{Geometry-based single-view mask denoising.}
We enforce two geometry-based distance constraints to refine each 3D mask within a single view, where $\delta$ is the distance threshold defined for the point cloud space.
In the objective of Eq.~(4), the first constraint is the \emph{intra-instance continuity} that ensures that each mask forms a spatially coherent manifold to avoid covering points from non-adjacent instances.
To implement the intra-instance continuity, we directly employ the density-based clustering method DBSCAN to filter out points that do not meet the constraint $\forall i \in {m}_a^t,\exists j \in {m}_a^t\ \vert\ dis(p_{(i)},p_{(j)}) < \delta$.
Specifically, we set the $neighborhood\_radius$ as $\delta$ and set $min\_samples$ $1$ using Euclidean metric.
After processing one mask, multiple masks might be obtained and we select the biggest one as the final refined mask.
Meanwhile, the second constraint is the \emph{inter-instance separation}, which mitigates the point cloud boundary noise caused by projection errors.
To implement the inter-instance separation, we use the point cloud data structure in the open3D library to achieve rapid calculation of the distance between points in two masks.
For points that do not satisfy the distance constraint $\forall i \in {m}_a^t,\forall j \in {m}_b^t\ \vert\ dis(p_{(i)},p_{(j)}) > \delta$ in the 3D point cloud space, we exclude them in their original mask.

\textbf{SCP-based multi-view noise mask filtering.}
To implement this, we solve the following set covering problem (SCP) to select the key mask set $\mathcal{M}$ from the class-agnostic mask set $\mathcal{C}^{(t-k+1):t}$:
\begin{equation}
\begin{aligned}
&\arg\min_{\mathcal{M} \subseteq \mathcal{C}^{(t-k+1):t}} |\mathcal{M}| \\ & s.t.~\forall p \in \mathcal{K}^{(t-k+1):t}, \exists m \in \mathcal{M}~\vert~p \in m, \nonumber
\end{aligned}
\end{equation}
where $m$ denotes a key mask in $\mathcal{M}$ and $p$ is the index of a key point from $\mathcal{K}^{(t-k+1):t}$.
Since the set covering problem is NP-hard, we adopt a greedy algorithm to approximate a solution as shown in the following Algorithm~\ref{algorithm}.
Specifically, we reorder the masks in $\mathcal{C}^{(t-k+1):t}$ from large to small by mask size, and then prioritize adding the mask that covers the most new points to $\mathcal{M}$ until all points in $\mathcal{K}^{(t-k+1):t}$ are covered.
As a result, our method can select necessary 3D masks for segmentation from a large set of redundant ones, which is conducive to filtering out noise masks across multiple views and reducing the chances of noise interference.
\begin{algorithm}[!h]
\caption{Greedy solving strategy for set-covering-problem based multi-view noise mask filtering}
\label{algorithm}
\begin{algorithmic}[1]
\STATE \textbf{Input:} mask set $\mathcal{C}^{(t-k+1):t}$, key point set $\mathcal{K}^{(t-k+1):t}$
\STATE \textbf{Output:} selected mask set $\mathcal{M}\subseteq \mathcal{C}^{(t-k+1):t}$ that covers all points in the key point set $\mathcal{K}^{(t-k+1):t}$
\STATE Compute mask sizes: for $m\in\mathcal{C}^{(t-k+1):t},s(m)=|m|$
\STATE Sort masks in descending size order: \\ $L \leftarrow \mathcal{C}^{(t-k+1):t}$ such that $|L(i)| \geq |L(j)|, i<j$
\STATE $\mathcal{M}\leftarrow\emptyset$
\STATE $U \leftarrow \mathcal{K}^{(t-k+1):t}$ /*set of uncovered points*/
\WHILE{$U \neq \emptyset$ and $L \neq \emptyset$}
  \FOR{each mask $m$ in $L$}
    \STATE compute marginal gain $g(m)\leftarrow |m \cap U|$
  \ENDFOR
  \STATE select $m^\ast \leftarrow \arg\max_{m\in L} g(m)$
  \IF{$g(m^\ast)=0$}
    \STATE break
  \ENDIF
  \STATE add $m^\ast$ to $\mathcal{M}$
  \STATE update uncovered set: $U \leftarrow U \setminus m^\ast$
  \STATE remove $m^\ast$ from $L$
\ENDWHILE
\STATE \Return $\mathcal{M}$
\end{algorithmic}
\end{algorithm}

\textbf{Closed-form solution for set partitioning.}
The constraints in Eq.~(7) separate across coarse-grained class-agnostic masks and fine-grained semantic masks, because the semantic masks $\{m_s^\Delta\}$ are mutually disjoint and the assignment $g^*(s)$ is deterministic.
Consequently, the objective $\mathcal{J}$ can be maximized by independently optimizing the contribution of each mask.
The globally optimal solution that maximizes $\mathcal{J}$ is therefore obtained in the following closed form:
\[
\begin{aligned}
z_g^* &=
\begin{cases}
1, & \text{if } |\mathcal{C}_g|>0,\ S_g \ge |m_g^\Delta|,
% 1, & \text{if } |\mathcal{C}_g|>0,\ \dfrac{S_g}{|m_g^\Delta|} \ge 1-\tau,\\[4pt]
   ~\text{and } S_g + \lambda|\mathcal{C}_g| \ge |m_g^\Delta|,\\[8pt]
0, & \text{otherwise},
\end{cases}\\[5pt]
w_s^* &=
\begin{cases}
1, & \text{if } g^*(s)=\text{None},\ m_s^\Delta \cap \big(\bigcup_g m_g^\Delta\big)=\emptyset,
   ~\text{and } \forall g\,(g^*(s)=g \Rightarrow z_g^*=0),\\[8pt]
0, & \text{otherwise}.
\end{cases}
\end{aligned}
\]
The condition for $z_g^*=1$ requires that (a) $m_g^\Delta$ possesses at least one assigned cropped part,
% the parts cover at least $1-\tau$ of its points, 
(b) the parts cover its original points (here, we can choose to apply the more relaxed condition ${S_g}/{|m_g^\Delta|} \ge 1-\tau$ to cover at least $1-\tau$ of points, \eg $\tau=0.2$ in practical experiments), 
and (c) the decomposition reward $S_g+\lambda|\mathcal{C}_g|$ is no less than that of retaining the original $|m_g^\Delta|$ points.
The condition for $w_s^*=1$ keeps a semantic mask as a novel instance only when it has zero overlap with any coarse-grained mask, and is not consumed by any decomposition.
This closed-form solution directly attains the global maximum of $\mathcal{J}$, as each $z_g^*$ selects the strictly better option under the coverage constraint, as well as each $w_s^*$ includes every eligible novel instance without conflict.

\subsection{Complexity Analysis}
First, we present the complexity analysis for the proposed manifold refinement (which is the most heavy part in our design), indicating that our method still exhibits a linear growth with respect to the number of points.
Then, we present an analysis of the computational advantages of using detection results (bounding boxes) for mask retrieval and instance location.

\textbf{Complexity analysis for manifold refinement.}
The computational cost of the proposed manifold refinement mainly comprises two stages:
(i) construction of the super‑point graph $\mathcal{G}$,
and (ii) iterative Bellman relaxation of the Eikonal equations.
We analyze each stage in terms of the number of input points
$n^{\Delta}=|\mathcal{I}^{\Delta}|$, the number of super‑points
$V=|\mathcal{V}|$, the number of undirected edges $E=|\mathcal{E}|$, and the
number of masks $C^{\Delta}$.

Specifically, in the super‑point graph construction, computing the average nearest-neighbour distance requires one $k$‑nearest‑neighbour query per point ($k=2$); and using a $3$‑dimensional \textsc{kd}‑tree this takes $O(n^{\Delta}\log n^{\Delta})$ time~\cite{friedman1977algorithm}. Partitioning the $n^{\Delta}$ points into voxels is performed in a single linear pass $O(n^{\Delta})$.
For each of the $V$ super‑points we compute the centroid and the $3\times 3$ covariance matrix in Eq.~(9) together with its inverse; and since every point contributes to exactly one super‑point, the total cost
of these operations is $O(n^{\Delta})$ (the $3\times 3$ matrix inversion adds only a constant factor per super‑point).
Building the edge set $\mathcal{E}$ requires a radius search around each super‑point centroid; and employing a spatial index (\eg a \textsc{kd}‑tree) on the $V$ centroids bounds this step by $O(V\log V)$.
Because the search radius is proportional to the local point spacing, the degree of each node remains bounded under the assumption of locally uniform sampling; and consequently the total number of edges $E$ is $O(V)$.
Computing the Mahalanobis edge weight in Eq.~(10) for each edge introduces an additional $O(E)$ arithmetic operations.

In iterative Bellman relaxation, one iteration over all $V$ super‑points evaluates the operation $\min_{u\in\mathcal{N}(v)}\{\phi_c^{(r-1)}(u)+d_{\mathcal{M}}(v,u)\}$ for each mask index $c$.
This requires aggregating information from the neighbours of every node, whose total count over all nodes equals $2E$.
Hence a single iteration for all $C^{\Delta}$ masks costs $O\bigl(C^{\Delta}\,E\bigr)$. With $R$ iterations account for $O\bigl(R\,C^{\Delta}\,E\bigr)$ operations.
In our experiments, $R$ is fixed to $5$ and thus the cost time is negligible compared to the initial nearest‑neighbour search.
After finished, assigning each originally point $i\in\mathcal{U}^{\Delta}$ to the mask of its enclosing super‑point is a direct lookup that runs in $O(n^{\Delta})$ time.

Summing above two stages, the total time complexity is
\begin{equation}
O\bigl(n^{\Delta}\log n^{\Delta} + V\log V + R\,C^{\Delta}\,E + n^{\Delta}\bigr). \nonumber
\end{equation}
Because $V\ll n^{\Delta}$ and $E=O(V)$, the term $O(n^{\Delta}\log n^{\Delta})$ dominates, which stems from the nearest‑neighbour distance estimation that is already required by the preceding merging steps.
Therefore, the manifold refinement adds only a small constant overhead to the complete processing pipeline.
The space complexity amounts to $O(n^{\Delta}+V+E + C^{\Delta}V)$, where the leading term $O(n^{\Delta})$ accounts for storing the original point cloud and the auxiliary structures occupy only a fraction of that memory.

\textbf{Efficiency for detection-based fast mask localization.}
In our \emph{local-to-historical 3D streaming instance segmentation}, we employ detection-based mask localization to achieve fast mask retrieval.
We give the following complexity analysis to illustrate the computational efficiency.

Specifically, let $n_h^t$ and $n_l^t$ be the numbers of masks in the historical 3D instance mask set $M^{1:t-k}$ and in the local 3D instance mask set $M^{(t-1+k):t}$ at the frame step $t$, respectively.
For the convenience of analysis, let $k_t$ denote the average number of points per mask at frame step $t$, and let $\alpha_t\in [0,1]$ denote the number proportion of the masks in $M^{(t-1+k):t}$ that actually have overlapping points with the masks in $M^{1:t-k}$ at frame step $t$.
For the normal mask localization strategy (\emph{a.k.a,} Normal MLS),
each mask pair in $M^{1:t-k} \times M^{(t-1+k):t}$ is tested by comparing every element of one list to every element of the other, costing $O(k_t^2)$ per pair,
therefore the time cost at frame step $t$ is
\begin{equation}
\begin{aligned}
O_{\text{normal}}(t)=O\big(n_h^t\cdot n_l^t\cdot k_t^2\big), \nonumber
\end{aligned}
\end{equation}
and the total cost over $T$ frames is
\begin{equation}
\begin{aligned}
O_{\text{normal}}({1:T})=\sum_{t=1}^{T}O\big(n^t_h\cdot n^t_l\cdot k_t^2\big). \nonumber
\end{aligned}
\end{equation}

For our detection-based mask localization strategy (\ie Detection-based MLS), each pair first incurs a 3D space intersection test at constant cost $O(1)$,
only the $\alpha_t\cdot n^t_h\cdot n^t_l$ surviving pairs undergo the naive $O(k_t^2)$ cost.
The cost at frame step $t$ is thus
\begin{equation}
\begin{aligned}
O_{\text{bbox}}(t)&=O\big(n^t_h\cdot n^t_l\cdot 1\big)\;+\;O\big(\alpha_t\cdot n^t_h\cdot n^t_l\cdot k_t^2\big)  \\
      &=O\big(n^t_h\cdot n^t_l + \alpha_t\cdot n^t_h\cdot n^t_l\cdot k_t^2\big), \nonumber
\end{aligned}
\end{equation}
and the cumulative cost over $T$ frames is
\begin{equation}
\begin{aligned}
O_{\text{bbox}}({1:T})=\sum_{t=1}^{T}O\big(n^t_h\cdot n^t_l + \alpha_t\cdot n^t_h\cdot n^t_l\cdot k_t^2\big). \nonumber
\end{aligned}
\end{equation}
If $\alpha_t\ll 1$ for most frames $t$, then $O_{\text{bbox}}({1:T})\approx \sum_{t=1}^TO(n^t_h\cdot n^t_l) \ll O_{\text{normal}}({1:T})$ and it yields an asymptotic improvement of roughly a factor $O(k_t^2)$ over the Normal MLS.
The per‑mask bounding box construction (min/max over $x,y,z$ in 3D space) costs $O(k_t)$ and could be counted as minor costs in dynamic scenarios.
As a result, the improvement of time efficiency by our method (with Detection-based MLS) becomes more significant, especially as the number of frames and masks increases.
We have also conducted ablation experiments to verify the effectiveness of our Detection-based MLS in the main paper.

\subsection{Convergence Analysis}
We prove that the iteration in Eq.~(12) converges in a finite number of steps to the exact shortest-path distance on the anisotropic graph in Eq.~(11).
The argument hinges on the monotonicity of the update and the Bellman-Ford principle~\cite{bellman1958routing} for shortest‑path computations on a finite graph.

\subsubsection{Preliminaries and notation}
For each mask index $c \in \{1,\dots,C^{\Delta}\}$, let $\mathcal{S}_c = \{ v \in \mathcal{V} \mid \mathcal{I}_v^{\Delta} \cap m_c \neq \emptyset \}$ be the set of source super‑points.
The \emph{true} geodesic distance from $\mathcal{S}_c$ under the anisotropic metric $d_{\mathcal{M}}$ is the function $\phi_c^* : \mathcal{V} \to \mathbb{R}_{\ge 0} \cup \{\infty\}$ defined by
\begin{equation}
\phi_c^*(v) = \inf_{\gamma \in \Gamma(v,\mathcal{S}_c)} \sum_{(u,w)\in\gamma} d_{\mathcal{M}}(u,w), \nonumber
\label{eq:geodesic}
\end{equation}
where $\Gamma(v,\mathcal{S}_c)$ denotes the set of all finite paths in $\mathcal{G}$ that connect $v$ to any node in $\mathcal{S}_c$ (with the convention $\inf \emptyset = \infty$).  It is straightforward to verify that $\boldsymbol{\phi}^* = (\phi_1^*,\dots,\phi_{C^{\Delta}}^*)$ satisfies Eq.~(11) and is the unique minimal solution.

The Bellman relaxation operator $F : \mathbb{R}_{\ge 0}^{C^{\Delta}\times V} \to \mathbb{R}_{\ge 0}^{C^{\Delta}\times V}$ that performs one complete pass of Eq.~(12) is defined component‑wise as
\begin{equation}\tag{A}
[F(\boldsymbol{\phi})]_c(v) =
\begin{cases}
0,            \;\;\;\;\;\;\;\;\;\;\;\;\;\;\;\;\;\;\;\;\;\;\;\;\;\;\;\;\;\;\;\;\;\;\;\;\;\;\;\;\;\;\;\;\;\;\;\;\;\;\;\;\;\;\;\;\;\;\;\; v \in \mathcal{S}_c,\\[8pt]
\begin{aligned}
\min\!\Bigl\{&\phi_c(v),\min_{u\in\mathcal{N}(v)} \bigl( \phi_c(u) + d_{\mathcal{M}}(v,u) \bigr) \Bigr\},\ v \notin \mathcal{S}_c.
\end{aligned}
\end{cases}
\end{equation}
The iteration is $\boldsymbol{\phi}^{(r+1)} = F(\boldsymbol{\phi}^{(r)})$, initialised with $\phi_c^{(0)}(v) = 0$ for $v \in \mathcal{S}_c$ and $+\infty$ elsewhere.

\subsubsection{Monotone convergence to the fixed point}
The space $\mathcal{X} = \mathbb{R}_{\ge 0}^{C^{\Delta}\times V}$ is
partially ordered by the component‑wise relation $\le$, and the operator
$F$ possesses two elementary properties that immediately yield convergence.

\begin{lemma}[Monotonicity]\label{lem:mono}
The Bellman relaxation operator $F$ is (a)~{order‑preserving}: $\boldsymbol{\phi} \le \boldsymbol{\psi} \;\Longrightarrow\; F(\boldsymbol{\phi}) \le F(\boldsymbol{\psi})$; and (b)~{deflationary}: $F(\boldsymbol{\phi}) \le \boldsymbol{\phi}$ for every $\boldsymbol{\phi} \in \mathcal{X}$.
\end{lemma}
\begin{proof}
Both properties follow directly from the above definition Eq.~(A): the update replaces the current value with the minimum of itself and a set of non‑negative candidates, therefore it never increases any component.
If $\boldsymbol{\phi} \le \boldsymbol{\psi}$, then for any neighbour $u$ we have $\phi_c(u) + d_{\mathcal{M}}(v,u) \le \psi_c(u) + d_{\mathcal{M}}(v,u)$, and taking minima preserves the inequality.
\end{proof}

\begin{lemma}[Lower bound]\label{lem:lower}
For every $r \ge 0$, $\boldsymbol{\phi}^{(r)} \ge \boldsymbol{\phi}^*$ component‑wise.
\end{lemma}
\begin{proof}
We use induction on $r$ and the claim holds for $r=0$ by initialisation. Assume $\boldsymbol{\phi}^{(r)} \ge \boldsymbol{\phi}^*$.  For any non‑source node $v$, we have
\begin{align}\label{eq:lowerbound}
[F(\boldsymbol{\phi}^{(r)})]_c(v)
&= \min\!\Bigl\{ \phi_c^{(r)}(v),\; \min_{u\in\mathcal{N}(v)} \bigl( \phi_c^{(r)}(u) + d_{\mathcal{M}}(v,u) \bigr) \Bigr\} \nonumber \\
&\ge \min\!\Bigl\{ \phi_c^*(v),\; \min_{u\in\mathcal{N}(v)} \bigl( \phi_c^*(u) + d_{\mathcal{M}}(v,u) \bigr) \Bigr\} \nonumber \\ 
&= \phi_c^*(v), \nonumber
\end{align}
where the inequality uses the induction hypothesis and the fact that addition and $\min$ are monotone.  Source nodes equal $0$ for both sides.
\end{proof}

From Lemma~\ref{lem:mono} the sequence $\{\boldsymbol{\phi}^{(r)}\}$ is component‑wise non‑increasing, and from Lemma~\ref{lem:lower} it is bounded below by $\boldsymbol{\phi}^*$.
Hence the pointwise limit $\boldsymbol{\phi}^{(\infty)} = \lim_{r\to\infty} \boldsymbol{\phi}^{(r)}$ exists.
Because each $\phi_c(v)$ takes values in the finite set of path lengths from $\mathcal{S}_c$ (the graph is finite and edge weights are fixed), the decreasing sequence must become constant after a finite number of
steps.
Consequently, there exists $R_0$ such that for all $t \ge T_0$, $\boldsymbol{\phi}^{(r)} = \boldsymbol{\phi}^{(\infty)}$, and the limit satisfies $\boldsymbol{\phi}^{(\infty)} = F(\boldsymbol{\phi}^{(\infty)})$.

\begin{theorem}[Fixed‑point optimality]
\label{thm:fixed}
$\boldsymbol{\phi}^{(\infty)} = \boldsymbol{\phi}^*$.
\end{theorem}
\begin{proof}
The fixed‑point condition $\boldsymbol{\phi}^{(\infty)} = F(\boldsymbol{\phi}^{(\infty)})$
implies for every non‑source $v$ and every mask index $c$ that
\[
\phi_c^{(\infty)}(v) = \min_{u\in\mathcal{N}(v)} \bigl\{ \phi_c^{(\infty)}(u) + d_{\mathcal{M}}(v,u) \bigr\},
\]
because the value $\phi_c^{(\infty)}(v)$ cannot be strictly smaller than the minimum over its neighbours (otherwise the deflationary update would have reduced it).  Together with $\phi_c^{(\infty)}(v)=0$ on $\mathcal{S}_c$,
this is precisely the discrete Eikonal system in Eq.~(11).  Since $\boldsymbol{\phi}^*$ is the unique solution of that system, we obtain $\boldsymbol{\phi}^{(\infty)} = \boldsymbol{\phi}^*$.
\end{proof}

\subsubsection{Finite‑step convergence via dynamic programming}
The above analysis guarantees eventual convergence; and the Bellman-Ford perspective~\cite{bellman1958routing} refines this to a concrete bound on the number of iterations.
Let $D$ be the \emph{hop diameter} of $\mathcal{G}$ with respect to the source sets, \ie the maximum number of edges on a simple shortest path from any node to any source.
A straightforward induction on the path length shows that after $r$ iterations, we have
\[
\phi_c^{(r)}(v) = \phi_c^*(v)
\]
for all $v$ whose geodesic path from $\mathcal{S}_c$ uses $\le r$ edges. Since no optimal path exceeds $D$ edges, $R \ge D$ iterations are sufficient for exact convergence of all distance fields.

In our point‑cloud pipeline, the super‑point graph $\mathcal{G}$ is exceptionally compact, leading to a diameter $D$ typically smaller than $\sim5$.
In experiments, we have the fixed choice $R = 5$ as a default setting for every processed scene.
Disconnected components are handled seamlessly: nodes unreachable from any source retain $\phi_c = \infty$ and are assigned the sentinel label $-1$, as prescribed by Eq.~(14).

\subsubsection{Remarks on iteration order}
The convergence proof does not rely on the order in which nodes are visited, and any permutation gives the same final result.
Our implementation uses the natural order induced by the voxel grid, which accelerates information propagation along spatial directions, but the theoretical guarantees remain unaffected.

\newpage

\subsection{More Experimental Results}
\begin{table*}[!h]
\caption{\textbf{Per-class segmentation performance comparison} (semantic AP$_{50}$ across the first 20 classes) on ScanNet++ benchmark.
}\label{table:pcpp}
\centering
\renewcommand\tabcolsep{3.0pt}
\resizebox{\linewidth}{!}{
    \begin{tabular}{l|cccccccccccccccccccc|c}
    \toprule[0.7pt]
    
    & \rotatebox{90}{door}   & \rotatebox{90}{table}     & \rotatebox{90}{cabinet} & \rotatebox{90}{ceiling lamp}   & \rotatebox{90}{curtain}   & \rotatebox{90}{chair}  & \rotatebox{90}{blinds}
    & \rotatebox{90}{storage cabinet}  & \rotatebox{90}{bookshelf}    & \rotatebox{90}{office chair}  & \rotatebox{90}{window} & \rotatebox{90}{whiteboard}    & \rotatebox{90}{monitor} & \rotatebox{90}{shelf}
    & \rotatebox{90}{window frame}  & \rotatebox{90}{pipe}    & \rotatebox{90}{box}  & \rotatebox{90}{heater}    & \rotatebox{90}{kitchen cabinet}    & \rotatebox{90}{storage rack} & \rotatebox{90}{\textbf{Avg.}}
    \\ 
    \midrule
        MaskClustering
        & 0.0  & 6.5 & 0.0
        & 0.8  & 0.0 & 22.2
        & 0.0  & 0.0 & 16.0
        & 1.3  & 0.0 & 22.2
        & 0.0  & 0.0 & 0.0
        & 0.0  & 0.0 & 0.0
        & 0.0  & 0.0 & 3.5
        \\
        Stream3D
        & 1.4  & 3.1  & 0.0
        & 0.0  & 11.1 & 27.6
        & 0.0  & 0.0  & 22.2
        & 4.3  & 1.7  & 22.2
        & 0.0  & 0.0  & 0.0
        & 0.0  & 0.0  & 0.0
        & 0.0  & 0.0  & 4.7
        \\
        Stream3Dv2
        & 56.2  & 26.4 & 17.0
        & 48.8  & 30.3 & 52.0
        & 32.2  & 8.9  & 53.5
        & 38.9  & 10.5 & 45.3
        & 39.8  & 23.4 & 8.3
        & 15.3  & 18.1 & 0.0
        & 13.9  & 0.0  & 26.9
        \\
        \midrule
        \textbf{Improve. $\uparrow$}
        & 54.8  & 19.9 & 17.0
        & 48.0  & 19.2 & 24.4
        & 32.2  & 8.9 & 31.3
        & 34.6  & 8.8 & 23.1
        & 39.8  & 23.4 & 8.3
        & 15.3  & 18.1 & 0.0
        & 13.9  & 0.0 & \textbf{22.2}
        \\
    \bottomrule[0.7pt]
    \end{tabular}
}
\end{table*}

\begin{table*}[!h]
\caption{\textbf{Per-class segmentation performance comparison} (semantic AP$_{50}$ across the first 20 classes) on MatterPort3D benchmark.
}\label{table:pcma}
\centering
\renewcommand\tabcolsep{3.0pt}
\resizebox{\linewidth}{!}{
    \begin{tabular}{l|cccccccccccccccccccc|c}
    \toprule[0.7pt]
    
    & \rotatebox{90}{door}   & \rotatebox{90}{picture}     & \rotatebox{90}{window} & \rotatebox{90}{chair}   & \rotatebox{90}{pillow}   & \rotatebox{90}{lamp}  & \rotatebox{90}{cabinet}
    & \rotatebox{90}{curtain}  & \rotatebox{90}{table}    & \rotatebox{90}{plant}  & \rotatebox{90}{mirror} & \rotatebox{90}{towel}    & \rotatebox{90}{sink} & \rotatebox{90}{shelves}
    & \rotatebox{90}{sofa}  & \rotatebox{90}{bed}    & \rotatebox{90}{night stand}  & \rotatebox{90}{toilet}    & \rotatebox{90}{column}    & \rotatebox{90}{banister} & \rotatebox{90}{\textbf{Avg.}}
    \\ 
    \midrule
        MaskClustering
        & 9.3   & 2.9  & 5.2
        & 24.9  & 16.7 & 4.3
        & 0.0   & 14.1 & 3.3
        & 14.8  & 3.1  & 33.3
        & 15.6  & 0.0  & 5.4
        & 22.5  & 8.9  & 11.2
        & 29.5  & 12.8 & 11.9
        \\
        Stream3D
        & 10.1  & 3.2  & 7.7
        & 26.0  & 15.8 & 6.7
        & 8.9   & 9.7  & 3.2
        & 19.4  & 5.7  & 24.7
        & 8.6   & 0.0  & 11.5
        & 31.7  & 15.3 & 38.6
        & 27.4  & 0.0  & 13.7
        \\
        Stream3Dv2
        & 19.6  & 15.3 & 27.0
        & 26.3  & 17.2 & 53.4
        & 18.9  & 56.5 & 7.3
        & 50.1  & 5.8  & 36.1
        & 28.1  & 2.4  & 30.9
        & 56.6  & 27.3 & 62.2
        & 51.1  & 50.8 & 32.1
        \\
        \midrule
        \textbf{Improve. $\uparrow$}
        & 9.5  & 12.1 & 19.3
        & 0.3  & 0.5 & 46.7
        & 10.0  & 42.4 & 4.0
        & 30.7  & 0.1 & 2.8
        & 12.5  & 2.4 & 19.4
        & 24.9  & 12.0 & 23.6
        & 21.6  & 38.0 & \textbf{18.4}
        \\
    \bottomrule[0.7pt]
    \end{tabular}
}
\end{table*}
\begin{figure*}[!h]
\centering
\includegraphics[width=1\linewidth]{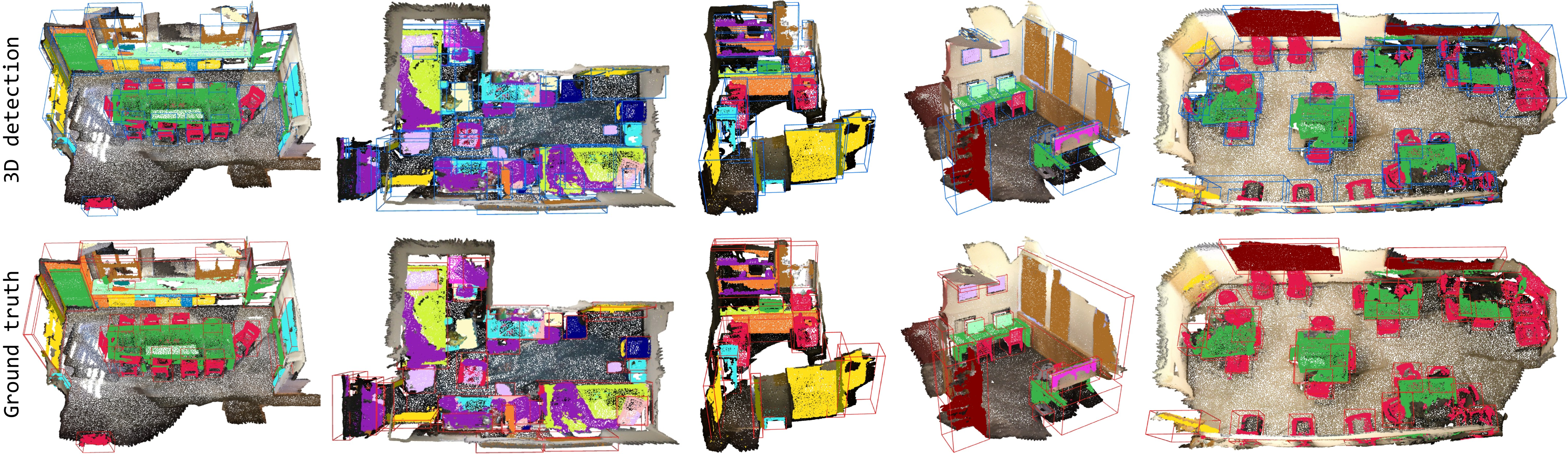}
\caption{More 3D object detection cases on ScanNet200 benchmark.}\label{fig:3DOD}
\end{figure*}
\begin{figure}[!t]
\centering
  \begin{subfigure}{0.3\linewidth}
    \includegraphics[width=\linewidth]{fig/C_Local_K.png}
    % \caption{Class-agnostic}
  \end{subfigure}
  \begin{subfigure}{0.3\linewidth}
    \includegraphics[width=\linewidth]{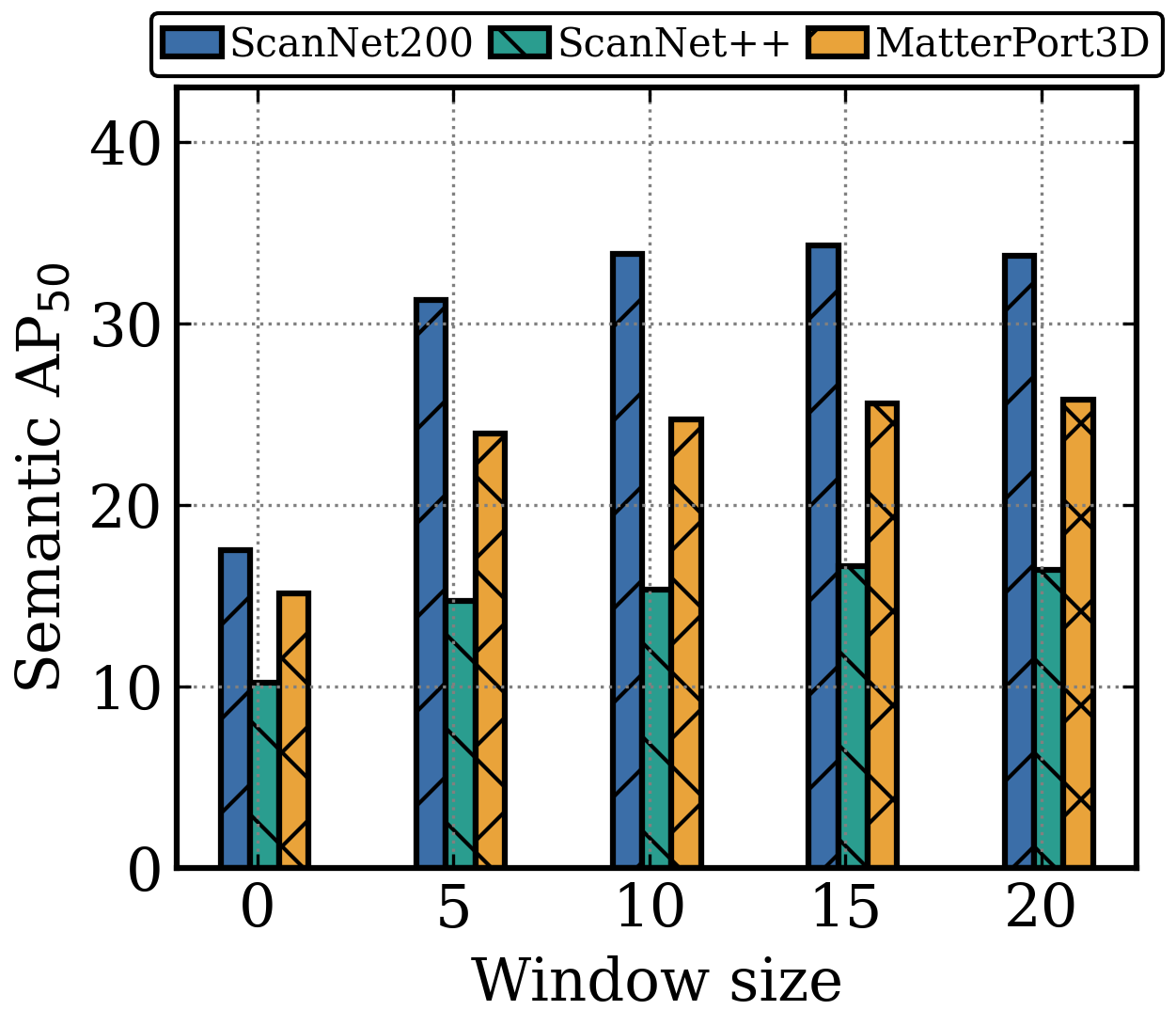}
    % \caption{Semantic}
  \end{subfigure}
\caption{Parameter analysis of the number of local frames $k$.}\label{frames}
\end{figure}
\begin{figure}[!t]
\centering
  \begin{subfigure}{0.3\linewidth}
    \includegraphics[width=\linewidth]{fig/C_delta.png}
    % \caption{Class-agnostic}
  \end{subfigure}
  \begin{subfigure}{0.3\linewidth}
    \includegraphics[width=\linewidth]{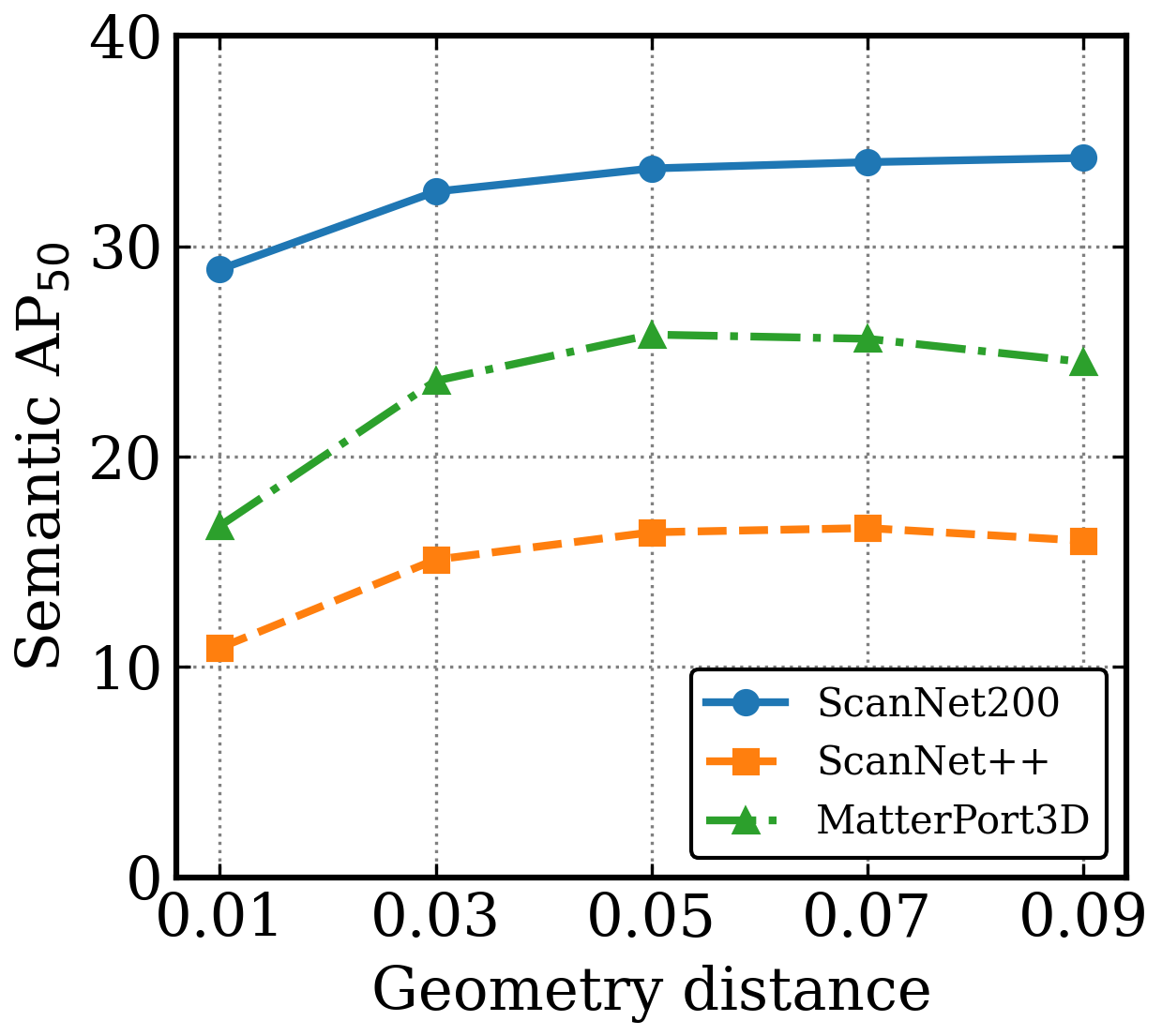}
    % \caption{Semantic}
  \end{subfigure}
\caption{Parameter analysis of the manifold distance $\delta$.}\label{manifold}
\end{figure}
\begin{figure}[!t]
\centering
  \begin{subfigure}{0.3\linewidth}
    \includegraphics[width=\linewidth]{fig/C_overlap.png}
    % \caption{Class-agnostic}
  \end{subfigure}
  \begin{subfigure}{0.3\linewidth}
    \includegraphics[width=\linewidth]{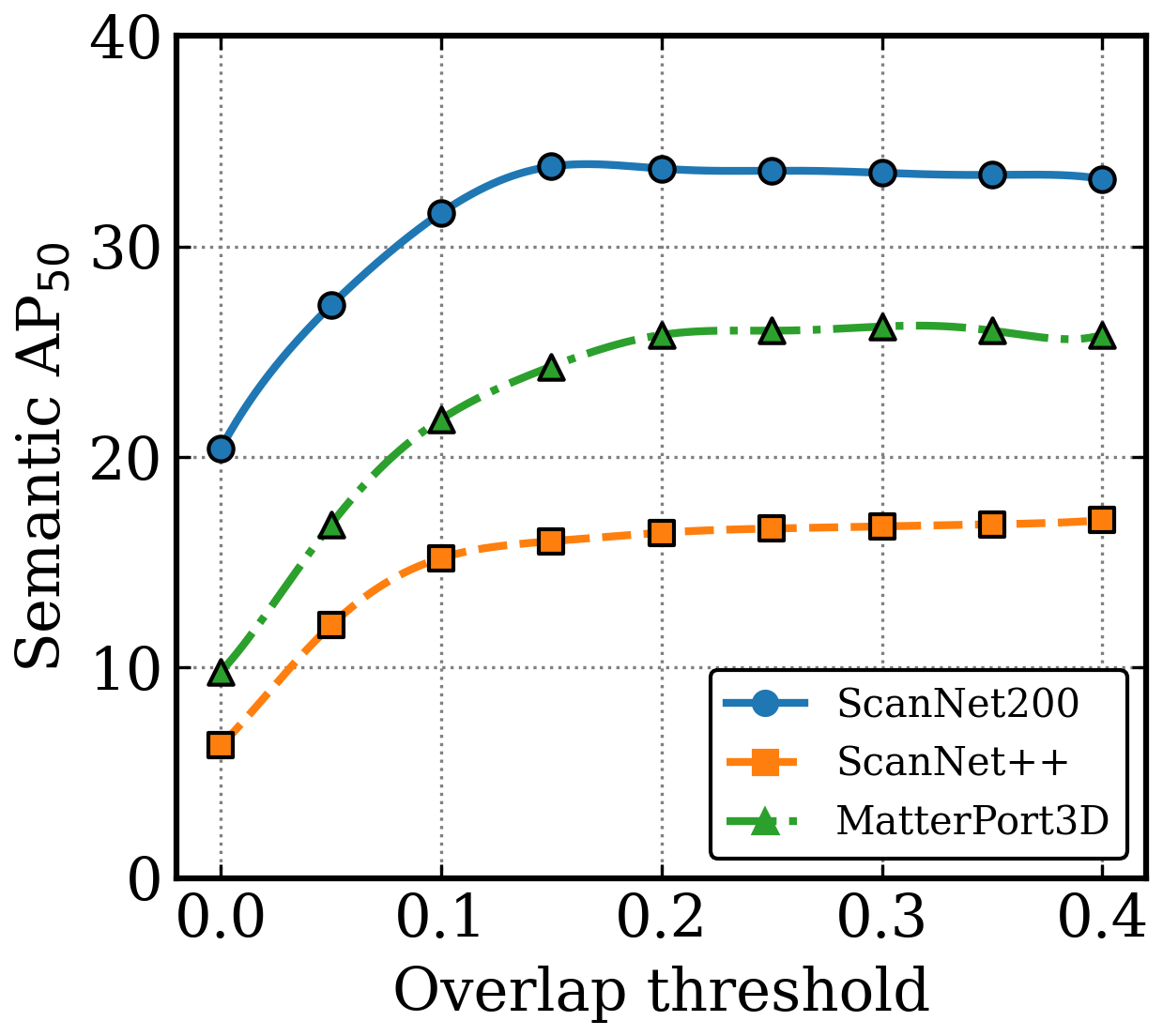}
    % \caption{Semantic}
  \end{subfigure}
\caption{Parameter analysis of the overlap threshold $\alpha$.}\label{overlap}
\end{figure}
Tables~\ref{table:pcpp} and \ref{table:pcma} show the per-class segmentation performance comparison evaluated by semantic AP$_{50}$ across the first 20 classes on ScanNet++ and MatterPort3D benchmarks, respectively.
Fig.~\ref{fig:3DOD} presents more 3D object detection results.
Fig.~\ref{frames}, \ref{manifold}, \ref{overlap} show three parameters' analysis evaluated by both class-agnostic and semantic AP$_{50}$ metrics.